\documentclass[pdflatex,sn-mathphys-num]{sn-jnl}

\usepackage{graphicx}%
\usepackage{multirow}%
\usepackage{amsmath,amssymb,amsfonts}%
\usepackage{amsthm}%
\usepackage{mathrsfs}%
\usepackage[title]{appendix}%
\usepackage{xcolor}%
\usepackage{textcomp}%
\usepackage{manyfoot}%
\usepackage{booktabs}%
\usepackage{algorithm}%
\usepackage{algorithmicx}%
\usepackage{algpseudocode}%
\usepackage{listings}%

\usepackage{tikz}
\usetikzlibrary{positioning,arrows.meta,calc,decorations.pathreplacing}
\usepackage{pifont}
\usepackage{arydshln}
\usepackage{caption}
\usepackage{graphicx}
\usepackage{rotating}

\usepackage[T1]{fontenc}

\theoremstyle{thmstyleone}%
\newtheorem{theorem}{Theorem}
\newtheorem{proposition}[theorem]{Proposition}%

\theoremstyle{thmstyletwo}%

\theoremstyle{thmstylethree}%

\begin{document}

\title[Article Title]{Solving Minimum Span Antibandwidth and Cyclic Antibandwidth Labeling Problems}


\author[1]{\fnm{Hieu} \sur{Truong Xuan}}\email{truongxuanhieu11@vnu.edu.vn}

\author*[1]{\fnm{Khanh} \sur{To Van}}\email{khanhtv@vnu.edu.vn}

\affil[1]{\centering
\orgdiv{Faculty of Information Technology},\\ \orgname{ VNU University of Engineering and Technology}, \\ \orgaddress{\street{Xuan Thuy, Cau Giay}, \city{Ha Noi}, \postcode{123105}, \country{Vietnam}}}


\abstract{The Antibandwidth and Cyclic Antibandwidth problems are NP-hard graph labeling problems that aim to maximize the minimum (cyclic) distance between labels assigned to adjacent vertices. 
Extensive research on these problems has resulted in a variety of mathematical formulations and computational approaches. However, their minimum span perspective, in which a prescribed minimum (cyclic) distance is fixed and the objective is to minimize the label span, has received comparatively little attention. 
In this paper, we consider this complementary perspective by introducing the Minimum Span Antibandwidth/Cyclic Antibandwidth Labeling (MSABL/MSCABL) problems and developing a unified Boolean Satisfiability (SAT)-based framework for solving them. 
The SAT-based framework formulates MSABL/MSCABL as a sequence of decision problems and exploits their monotonicity to accelerate the search process. 
We also consider two SAT solving strategies, parallel and incremental SAT solving: the former examines multiple candidate spans concurrently, while the latter reuses a single SAT instance while progressively restricting the label domain. The proposed approaches are evaluated on benchmark instances from the Harwell-Boeing Sparse Matrix Collection and compared with CPLEX\textsubscript{CP}, CPLEX\textsubscript{MIP}, and Gurobi. The results show that SAT-based approaches are highly competitive in solution quality, with the parallel approach performing best overall for MSCABL and the incremental approach for MSABL. With the no-hole constraint, they remain competitive with CPLEX\textsubscript{CP} and significantly outperform CPLEX\textsubscript{MIP} and Gurobi, particularly for MSCABL. These results demonstrate the effectiveness of SAT solving as an exact approach for MSABL and MSCABL.
}




\maketitle

\section{Introduction}
\label{sec:introduction}

Graph labeling problems have been extensively studied in combinatorial optimization and graph theory, with applications in areas such as frequency assignment, scheduling, and resource allocation. In a graph labeling problem, labels are assigned to vertices subject to constraints on the labels of adjacent or nearby vertices. A common objective is to maximize the separation between labels, thereby reducing interference between neighboring vertices. The Antibandwidth problem (ABP) is a representative example of this class of problems. In contrast to the classical Bandwidth problem, which minimizes the maximum distance between the labels of adjacent vertices, ABP maximizes the minimum distance between their labels \cite{leung1984some}.

Given an undirected graph, the ABP aims to assign a distinct label to each vertex such that the minimum distance between the labels of two adjacent vertices is maximized. This problem can be considered as ordering the vertices along a straight line so that adjacent vertices are placed as far apart as possible. The ABP has received considerable attention due to both its combinatorial difficulty and its applications to problems in which conflicting or adjacent objects should be assigned sufficiently different positions. Several theoretical studies have investigated ABP on structured graphs \cite{raspaud2009antibandwidth, wang2009explicit}, while more recent work has considered computational optimization and exact approaches on general graphs \cite{sinnl2021note, fazekas2020duplex, truong2025sequential}.

A natural extension of ABP is obtained by arranging the labels along a cycle instead of a straight line. This gives rise to the Cyclic Antibandwidth problem (CABP), in which the cyclic distance between two labels is measured by their minimum distance on the cycle, while the objective is to maximize the minimum cyclic distance over all edges. This problem has likewise been studied from both theoretical \cite{sykora2005cyclic, raspaud2009antibandwidth, wang2009explicit} and computational perspectives, with approximate \cite{bansal2011memetic,lozano2013hybrid, sundar2019hybrid,cavero2022general} and exact \cite{xuan2026solving} approaches being developed to obtain high quality solutions for challenging instances.

Both ABP and CABP are conventionally formulated as maximization problems. Specifically, the label domain is fixed, and the goal is to maximize the minimum (cyclic) distance between the labels of adjacent vertices. However, in many applications, the required (cyclic) distance between neighboring vertices may already be known, while the smallest label span required to realize it is the quantity of interest. This leads to a complementary perspective on the labeling problem, that is, instead of finding how large the minimum (cyclic) distance can be for a given label domain, we determine the smallest label span for which a prescribed minimum (cyclic) distance can be achieved.

In this work, we investigate a complementary minimum-span perspective for ABP and CABP. Given a prescribed minimum (cyclic) distance $k$, the Minimum Span Antibandwidth Labeling (MSABL) and Minimum Span Cyclic Antibandwidth Labeling (MSCABL) problems seek a feasible labeling with minimum label span satisfying the corresponding (cyclic) distance requirement. More precisely, for a candidate largest label $\lambda$, we consider labelings $f:V\rightarrow\{1,\ldots,\lambda\}$, with the label span defined as $\operatorname{span}(f)=\lambda-1$, and seek to minimize this value.

An important consequence of this problem is the monotonicity of the corresponding decision problem. Specifically, for a fixed (cyclic) distance $k$, the corresponding decision problem asks whether there is a feasible labeling for a given largest label $\lambda$. If a feasible labeling exists for $\lambda$, then a feasible labeling also exists for every $\lambda' \geq \lambda$ (see Appendix~\ref{appendix-monotonicity-property}). Therefore, the minimum span can be determined by identifying the smallest feasible value of $\lambda$ within a bounded range of candidate spans. This allows the optimization problem to be decomposed into a sequence of decision problems and makes it well suited to Boolean Satisfiability (SAT)-based solving, where the feasibility of each candidate span can be represented as a satisfiability instance.

The continuous development of Boolean Satisfiability (SAT) technology has enabled its application to a wide range of difficult combinatorial problems \cite{vardi2015sat}. Advances in both SAT algorithms \cite{vasconcellos2020abacus, haberlandt2023effective} and solver implementations \cite{Een2005MiniSat, Audemard2009Glucose, biere2024cadical, Biere2024Kissat} have further improved the ability of modern SAT solvers to handle large and complex instances. Furthermore, SAT has recently been successfully applied to both ABP and CABP. Existing SAT-based approaches exploit Staircase At-Most-One \cite{fazekas2020duplex, truong2025sequential} and Cyclic Ladder \cite{xuan2026solving} constraints induced by the (cyclic) distance requirements to construct compact propositional formulas. This provides a foundation for applying SAT to the minimum span variants considered in this work.

Motivated by these developments, we develop a SAT-based framework to solve both MSABL and MSCABL. This framework represents candidate spans as a sequence of SAT decision problems and exploits their monotonicity to identify the smallest feasible span. We also investigate two complementary solving strategies. The first is a parallel SAT approach, in which multiple spans can be solved independently, and searches that can no longer improve the current best solution are terminated. The second is an incremental SAT approach. When the largest label is decreased from $\lambda$ to $\lambda''$ ($\lambda'' < \lambda$), the SAT instance constructed for $\lambda$ can be reused by adding clauses that prohibit the use of labels from $\lambda'' + 1$ to $\lambda$. Consequently, the encoding does not need to be rebuilt at every iteration, and learned information from previous iterations can be retained. In this study, the incremental approach is applicable only to MSABL, as an incremental formulation for MSCABL is not yet available (see Appendix~\ref{sec:applicability-of-inc-sat-to-msl}).

The main contributions of this work can be summarized as follows:
\begin{itemize}
    \item \textit{We introduce the Minimum Span Antibandwidth Labeling (MSABL) and Minimum Span Cyclic Antibandwidth Labeling (MSCABL) problems as complementary minimum span formulations of ABP and CABP, respectively, and characterize their corresponding decision problems over candidate label domains.}

    \item \textit{We develop unified SAT-based frameworks for solving MSABL and MSCABL, exploiting the monotonicity of the corresponding decision problems to design an effective search strategy over candidate spans.}

    \item \textit{We investigate two complementary SAT solving strategies. The first is a parallel approach applicable to both MSABL and MSCABL, which exploits available hardware resources by solving multiple candidate spans concurrently. The second is an incremental approach for MSABL only, which reuses a single SAT instance by progressively adding constraints to exclude labels that are no longer available as the candidate largest label decreases.}

    \item \textit{We conduct an extensive computational evaluation on standard benchmark instances. The results demonstrate the effectiveness of the proposed SAT-based frameworks, which improve upon state-of-the-art optimization-based approaches by discovering better spans and proving optimality for many benchmark instances.}
\end{itemize}

The remainder of this paper is organized as follows. Section~\ref{sec:related-work} reviews existing research on ABP/CABP and discusses the research gap addressed in this work. Section~\ref{sec:problem-def-and-ilp-model} introduces MSABL/MSCABL and their Integer Linear Programming (ILP) models. Section~\ref{sec:sat-solving-for-msl} presents the SAT-based solving framework, including parallel and incremental SAT solving strategies, and describes the construction of SAT formulations based on the ILP model introduced in Section~\ref{sec:problem-def-and-ilp-model}. Section~\ref{sec:experimental-setup} presents the experimental setup and benchmark instances, while Section~\ref{sec:experimental-results} reports and discusses the computational results. Finally, Section~\ref{sec:conclusion} concludes the paper and discusses possible directions for future work.

\section{Related Work}
\label{sec:related-work}

In this section, we first review the development of the Antibandwidth problem and its computational approaches, followed by the corresponding literature on the Cyclic Antibandwidth problem. We then discuss the research gap that motivates the minimum span formulations considered in this work.

\subsection{Antibandwidth Problem}
\label{sec:related-abp}

The Antibandwidth problem (ABP) is a graph labeling problem that can be viewed as the counterpart of the classical Bandwidth problem. While the Bandwidth problem seeks to assign labels to adjacent vertices as close as possible \cite{chinn1982bandwidth}, ABP aims to maximize the minimum distance between the labels assigned to adjacent vertices. This problem has also been referred to as the separation problem \cite{leung1984some}, the dual bandwidth problem \cite{yixun2003dualbandwidth}, and the maximum differential coloring problem \cite{bekos2014note}.

Early research on the ABP focused primarily on its structural and theoretical properties. Raspaud et al.~\cite{raspaud2009antibandwidth} and Wang et al.~\cite{wang2009explicit} investigated the problem on several structured graph classes, including meshes and hypercubes. These results established several properties of this problem and provided a basis for subsequent computational studies.

From a computational perspective, Sinnl~\cite{sinnl2021note} investigated mathematical optimization formulations for the ABP. These formulations explicitly represent label assignment and optimize the minimum distance between labels of adjacent vertices. In addition to generating feasible labelings, the resulting optimization models provide a means of establishing optimality through exact solution procedures.

More recently, SAT has been explored as an alternative framework for solving the ABP. A central issue in SAT formulations is the efficient encoding of distance requirements between labels. Fazekas et al.~\cite{fazekas2020duplex} and Hieu et al.~\cite{truong2025sequential} investigated SAT encodings specifically designed for Staircase At-Most-One constraints arising from these requirements. Instead of encoding each constraint independently, their approach exploits the regular structure induced by consecutive labels to obtain a compact encoding.

\subsection{Cyclic Antibandwidth Problem}
\label{sec:related-cabp}

The Cyclic Antibandwidth Problem (CABP) extends the ABP by considering the cyclic distance between labels. While the ABP arranges labels along a linear domain, the CABP can be viewed as arranging them on a cycle and seeks to maximize the minimum cyclic distance between the labels assigned to adjacent vertices. This problem was first introduced by Leung et al.~\cite{leung1984some} as the cycle-separation problem. Subsequently, Sykora et al.~\cite{sykora2005cyclic} investigated its properties and established theoretical results for several structured graph classes. Further results for the problem were later reported by Raspaud et al.~\cite{raspaud2009antibandwidth}.

Several computational approaches have also been developed for CABP, with a particular emphasis on heuristic and metaheuristic methods. Bansal and Srivastava~\cite{bansal2011memetic} proposed a memetic algorithm for CABP. Lozano et al.~\cite{lozano2013hybrid} developed a hybrid artificial bee colony metaheuristic, while Sundar~\cite{sundar2019hybrid} proposed a hybrid ant colony optimization approach. More recently, Cavero et al.~\cite{cavero2022general} developed a general variable neighborhood search. These approaches are effective in obtaining high quality solutions, particularly for relatively large instances, but generally do not provide a proof of optimality.

Alongside heuristic and metaheuristic approaches, SAT-based approaches have also been investigated for solving the CABP. Compared with the linear case in the ABP, the cyclic formulation requires additional treatment of the wrap-around relationship between the first and last labels. Therefore, Hieu and Khanh \cite{xuan2026solving} developed a SAT-based approach that extends the encoding techniques proposed for the ABP \cite{truong2025sequential} to the cyclic structure.

\subsection{Research Gap}
\label{sec:research-gap}

The existing literature on ABP and CABP primarily considers these problems from a maximization perspective. Given a fixed label domain, the objective is to maximize the minimum (cyclic) distance between the labels assigned to adjacent vertices. Accordingly, existing heuristic and exact approaches are designed to approximate or determine the largest achievable (cyclic) distance value.

In this work, we consider the complementary minimum span perspective. Rather than maximizing (cyclic) distance $k$ for a fixed label domain, we fix the value of $k$ and minimize the size of the label domain. For the linear case, this gives the Minimum Span Antibandwidth Labeling (MSABL) problem, while the cyclic case gives the Minimum Span Cyclic Antibandwidth Labeling (MSCABL) problem.

This change in objective results in a different computational formulation. For a fixed (cyclic) distance $k$ and a largest label $\lambda$, the optimization problem can be formulated as a decision problem that asks whether there exists a feasible labeling whose labels belong to $\{1,\ldots,\lambda\}$. The feasibility of this decision problem is monotone with respect to $\lambda$ (see Appendix~\ref{appendix-monotonicity-property}). In particular, if a feasible labeling exists for a given $\lambda$, then a feasible labeling also exists for every larger value of $\lambda$. This property allows the minimum span problem to be solved through a sequence of SAT-based decision problems.

Although SAT-based methods have previously been developed for ABP and CABP, the use of SAT to solve their minimum span perspective has not been addressed in the existing literature. This work fills this gap by developing a SAT-based framework for both MSABL and MSCABL, combining the structured SAT encodings developed for ABP and CABP constraints with solving strategies that exploit the monotonicity of the minimum span formulation.

\section{Problem Definition and ILP Model}
\label{sec:problem-def-and-ilp-model}

\subsection{Minimum Span Labeling}
\label{sec:min-span-labeling}

Let $G=(V,E)$ be an undirected graph, where $V$ and $E$ denote the sets of vertices and edges of $G$, respectively. A labeling $f$ of $G$ is a mapping from $V$ to $\mathbb{Z}$, i.e.,
\[f:V\rightarrow\mathbb{Z}.\] 
The span of $f$ is defined as 
\[\operatorname{span}(f)=\max_{v\in V}f(v)-\min_{v\in V}f(v).\]
The objective of the Minimum Span Labeling problem is to find a feasible labeling that minimizes the difference between the largest and smallest assigned labels, i.e., the span of the labeling.

In most cases, adding a constant to every label preserves both the feasibility and the span of a labeling. As a result, it is commonly assumed in the literature, without loss of generality, that the minimum label is $0$ \cite{griggs1992labelling, junosza2013fast, junosza2013determining}. Under this normalization, a labeling can be represented as 
\[f:V\rightarrow\{0,1,\ldots,\lambda\},\] 
where $\lambda$ denotes the largest assigned label. Consequently, the span of the labeling is $\lambda$, and minimizing the span is equivalent to minimizing $\lambda$.

\subsection{Distance and Cyclic Distance constraint}

Let $G=(V,E)$ be an undirected graph, where $V$ and $E$ denote the sets of vertices and edges of $G$, and
\[
f:V\rightarrow\{1,2,\ldots,\lambda\}
\]
be a labeling of $G$, where $\lambda$ denotes the maximum assigned label.

\paragraph{Distance constraint.} 
For a given positive integer $k$, the distance constraint requires that the labels assigned to every pair of adjacent vertices $\{u, v\}$ differ by at least $k$. Let $D(u,v,f)$ denote the distance between the labels assigned to $u$ and $v$ under labeling $f$, i.e.,
\[
D(u,v,f)=|f(u)-f(v)|,
\]
the distance constraint is expressed as
\[
D(u,v,f) \ge k,\qquad \forall \{u,v\}\in E.
\]
Intuitively, adjacent vertices must be assigned labels that are sufficiently far apart on the linear ordering of labels. 

Figure~\ref{fig:ab_on_edge_u_v} illustrates the distance constraint on edge $\{u,v\}$ for $k=2$, assuming that vertex $u$ is assigned label 3. This figure presents two cases, in which the left case shows feasible labelings, whereas the right case shows infeasible labelings. In the feasible case, vertex $v$ is assigned either label 1 or 5, resulting in a distance of 2 between the labels assigned to $u$ and $v$ and thus satisfying the distance constraint. In contrast, in the infeasible case, $v$ is assigned either label 2 or 4, resulting in a distance of only 1 and violating the distance constraint.

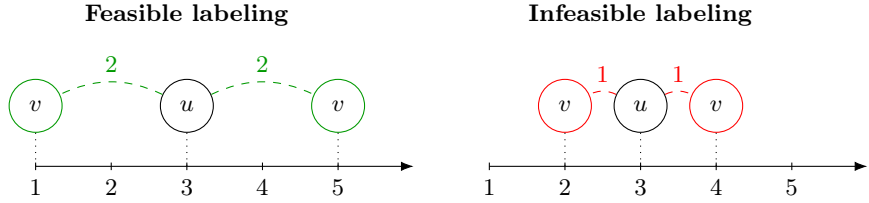
\begin{figure}[ht!]
\centering

\begin{tikzpicture}[
    vertex/.style={circle,draw,minimum size=7mm,inner sep=0pt,font=\small},
    every node/.style={font=\small},
    >=Latex
]






\node at (2,2) {\textbf{Feasible labeling}};

\draw[->] (0,0)--(5,0);

\foreach \x/\l in {0/1,1/2,2/3,3/4,4/5}
{
    \draw (\x,-0.05)--(\x,0.05);
    \node[below] at (\x,-0.05) {\l};
}

\node[vertex,draw=black] (uf) at (2,0.8) {$u$};
\node[vertex,draw=green!60!black] (vf1) at (0,0.8) {$v$};
\node[vertex,draw=green!60!black] (vf2) at (4,0.8) {$v$};

\draw[dashed,green!60!black]
(uf) to[bend left=25]
node[above] {$2$}
(vf2);

\draw[dashed,green!60!black]
(uf) to[bend right=25]
node[above] {$2$}
(vf1);

\draw[dotted] (uf)--(2,0);
\draw[dotted] (vf1)--(0,0);
\draw[dotted] (vf2)--(4,0);


\node at (8,2) {\textbf{Infeasible labeling}};

\draw[->] (6,0)--(11,0);

\foreach \x/\l in {6/1,7/2,8/3,9/4,10/5}
{
    \draw (\x,-0.05)--(\x,0.05);
    \node[below] at (\x,-0.05) {\l};
}

\node[vertex,draw=black] (ui) at (8,0.8) {$u$};
\node[vertex,draw=red] (vi1) at (7,0.8) {$v$};
\node[vertex,draw=red] (vi2) at (9,0.8) {$v$};

\draw[dashed,red]
(ui) to[bend right=25]
node[above] {$1$}
(vi1);

\draw[dashed,red]
(ui) to[bend left=25]
node[above] {$1$}
(vi2);

\draw[dotted] (ui)--(8,0);
\draw[dotted] (vi1)--(7,0);
\draw[dotted] (vi2)--(9,0);

\end{tikzpicture}

\caption{
Illustration of the antibandwidth constraint on an edge $\{u, v\}$ for $k = 2$ and $u$ is assigned label 3.
}
\label{fig:ab_on_edge_u_v}
\end{figure}

Figure~\ref{fig:ab_on_graph_5_vertices} presents a complete example of a labeling on a graph with 5 vertices for $k=2$. The left-hand side shows the labeling, while the right-hand side illustrates the corresponding graph. The number associated with each edge represents the distance between the labels of its vertices. Edges satisfying the distance constraint, i.e., having a distance of at least 2 between the labels of their vertices, are highlighted in green, whereas edges violating the constraint are highlighted in red. Since the edge $\{v_4, v_5\}$ has a distance of only 1, this labeling is infeasible.

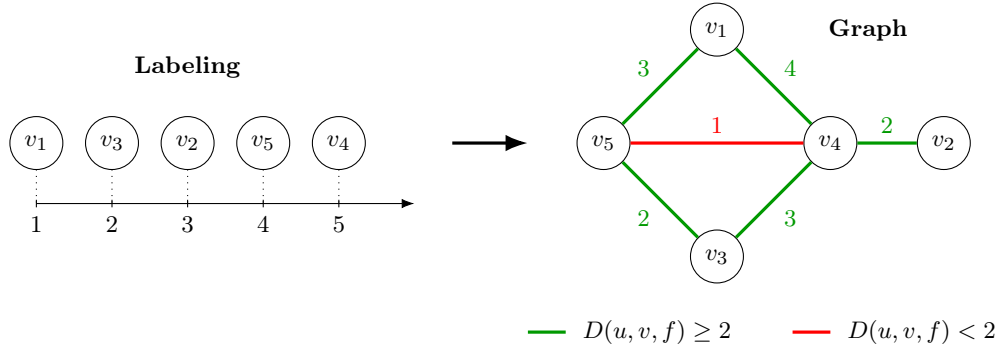
\begin{figure}[ht!]
\centering

\begin{tikzpicture}[
    vertex/.style={
        circle,
        draw,
        minimum size=7mm,
        inner sep=0pt,
        font=\small
    },
    >=Latex,
    every node/.style={font=\small}
]


\node at (2,1) {\textbf{Labeling}};

\draw[->] (0,-0.8)--(5,-0.8);

\foreach \x/\l in {0/1,1/2,2/3,3/4,4/5}
{
    \draw (\x,-0.85)--(\x,-0.75);
    \node[below] at (\x,-0.85) {\l};
}

\node[vertex] (L1) at (0,0) {$v_1$};
\node[vertex] (L3) at (1,0) {$v_3$};
\node[vertex] (L5) at (2,0) {$v_2$};
\node[vertex] (L2) at (3,0) {$v_5$};
\node[vertex] (L4) at (4,0) {$v_4$};

\foreach \x in {0,1,2,3,4}
{
    \draw[dotted] (\x,-0.35)--(\x,-0.8);
}


\draw[very thick,->] (5.5,0)--(6.5,0);


\node at (11,1.5) {\textbf{Graph}};

\node[vertex] (v1) at (9, 1.5) {$v_1$};
\node[vertex] (v5) at (7.5,0) {$v_5$};
\node[vertex] (v4) at (10.5,0) {$v_4$};
\node[vertex] (v3) at (9,-1.5) {$v_3$};
\node[vertex] (v2) at (12,0) {$v_2$};


\draw[green!60!black,very thick]
(v3)--node[below left]{$2$}(v5);

\draw[green!60!black,very thick]
(v5)--node[above left]{$3$}(v1);

\draw[green!60!black,very thick]
(v3)--node[below right]{$3$}(v4);

\draw[green!60!black,very thick]
(v4)--node[above right]{$4$}(v1);

\draw[green!60!black,very thick]
(v4)--node[above]{$2$}(v2);

\draw[red,very thick]
(v4)--node[above]{$1$}(v5);


\draw[green!60!black,very thick] (6.5,-2.5)--(7,-2.5);
\node[right] at (7.1,-2.5) {$D(u,v,f)\ge 2$};

\draw[red,very thick] (10,-2.5)--(10.5,-2.5);
\node[right] at (10.6,-2.5) {$D(u,v,f)< 2$};

\end{tikzpicture}

\caption{Illustration of the antibandwidth constraint on a graph with 5 vertices $\{v_1, \ldots, v_5\}$ for $k = 2$.}
\label{fig:ab_on_graph_5_vertices}
\end{figure}

\paragraph{Cyclic distance constraint.}
The cyclic distance constraint extends the distance constraint by arranging the labels on a cycle rather than on a line. Specifically, the cyclic distance constraint is defined as
\[
D_c(u,v,f) \ge k,\qquad \forall \{u,v\}\in E,
\]
where the cyclic distance between the labels assigned to $u$ and $v$ is denoted by
\[
D_c(u,v,f) = \min\bigl(|f(u)-f(v)|,\lambda-|f(u)-f(v)|\bigr),
\]
where $\lambda$ is the highest assigned label in $f$. 

Figure~\ref{fig:cab_on_edge_u_v} illustrates the cyclic counterpart of Figure~\ref{fig:ab_on_edge_u_v}, in which the labels are arranged in a cycle, and the dashed arcs represent the cyclic distance between the labels assigned to $u$ and $v$. As in Figure~\ref{fig:ab_on_edge_u_v}, the left example shows two feasible labelings with a cyclic distance of 2, whereas the right example shows two infeasible labelings with a cyclic distance of 1.

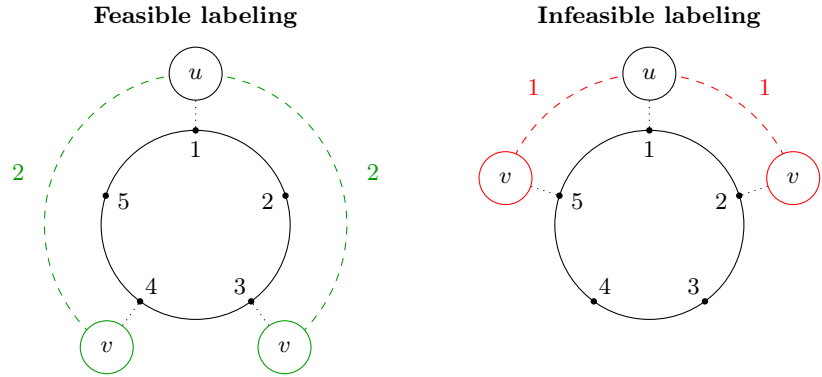
\begin{figure}[ht!]
\centering

\begin{tikzpicture}[
    vertex/.style={circle,draw,fill=white,minimum size=7mm,inner sep=0pt,font=\small},
    every node/.style={font=\small},
    >=Latex
]

\def\r{1.25}
\def\ra{2} %






\node at (2,2.75) {\textbf{Feasible labeling}};

\coordinate (CF) at (2,0);

\draw[-] ($(CF)+(0:\r)$)
arc[start angle=0,end angle=360,radius=\r];

\foreach \a/\l/\name in {90/1/P1,18/2/P2,-54/3/P3,-126/4/P4,162/5/P5}
{
    \coordinate (\name) at ($(CF)+(\a:\r)$);
    \fill (\name) circle (1.2pt);
}

\foreach \a/\l in {-90/1,-162/2,126/3,54/4,-18/5}
{
    \node at ($(CF)+(\a:-1)$) {\l};
}

\draw[dashed,green!60!black]
($(CF)+(90:\ra)$)
arc[start angle=90,end angle=-54,radius=\ra];
\node[right,green!60!black] at ($(CF)+(18:{\ra+0.25})$) {$2$};

\draw[dashed,green!60!black]
($(CF)+(90:\ra)$)
arc[start angle=90,end angle=234,radius=\ra];
\node[left,green!60!black] at ($(CF)+(162:{\ra+0.25})$) {$2$};

\node[vertex] (UF) at ($(CF)+(90:2)$) {$u$};

\node[vertex,draw=green!60!black]
(VF1) at ($(CF)+(234:2.0)$) {$v$};

\node[vertex,draw=green!60!black]
(VF2) at ($(CF)+(-54:2.0)$) {$v$};

\draw[dotted] (UF)--(P1);
\draw[dotted] (VF1)--(P4);
\draw[dotted] (VF2)--(P3);


\node at (8,2.75) {\textbf{Infeasible labeling}};

\coordinate (CI) at (8,0);

\draw[-] ($(CI)+(0:\r)$)
arc[start angle=0,end angle=360,radius=\r];

\foreach \a/\l/\name in {90/1/Q1,18/2/Q2,-54/3/Q3,-126/4/Q4,162/5/Q5}
{
    \coordinate (\name) at ($(CI)+(\a:\r)$);
    \fill (\name) circle (1.2pt);
}

\foreach \a/\l in {-90/1,-162/2,126/3,54/4,-18/5}
{
    \node at ($(CI)+(\a:-1)$) {\l};
}

\draw[dashed,red]
($(CI)+(90:\ra)$)
arc[start angle=90,end angle=162,radius=\ra];
\node[left,red] at ($(CI)+(126:{\ra+0.25})$) {$1$};

\draw[dashed,red]
($(CI)+(90:\ra)$)
arc[start angle=90,end angle=18,radius=\ra];
\node[right,red] at ($(CI)+(54:{\ra+0.25})$) {$1$};

\node[vertex] (UI) at ($(CI)+(90:2)$) {$u$};

\node[vertex,draw=red]
(VI1) at ($(CI)+(162:2.0)$) {$v$};

\node[vertex,draw=red]
(VI2) at ($(CI)+(18:2.0)$) {$v$};

\draw[dotted] (UI)--(Q1);
\draw[dotted] (VI1)--(Q5);
\draw[dotted] (VI2)--(Q2);

\end{tikzpicture}

\caption{Illustration of the cyclic antibandwidth constraint on an edge $\{u,v\}$ for $k=2$ and $u$ is assigned label~3.}
\label{fig:cab_on_edge_u_v}
\end{figure}

Similarly, Figure~\ref{fig:cab_on_graph_5_vertices} illustrates the cyclic counterpart of Figure~\ref{fig:ab_on_graph_5_vertices}. Each edge is annotated with the cyclic distance between the labels of its vertices. Edges satisfying the cyclic distance constraint are highlighted in green, whereas those violating the constraint are highlighted in red. In this example, the edges $\{v_1,v_4\}$ and $\{v_4,v_5\}$ have a cyclic distance of only 1, making the labeling infeasible.

\begin{figure}[ht!]
\centering

\begin{tikzpicture}[
    vertex/.style={
        circle,
        draw,
        minimum size=7mm,
        inner sep=0pt,
        font=\small
    },
    >=Latex,
    every node/.style={font=\small}
]


\node at (2,1) {\textbf{Labeling}};

\draw[->] (0,-0.8)--(5,-0.8);

\foreach \x/\l in {0/1,1/2,2/3,3/4,4/5}
{
    \draw (\x,-0.85)--(\x,-0.75);
    \node[below] at (\x,-0.85) {\l};
}

\node[vertex] (L1) at (0,0) {$v_1$};
\node[vertex] (L3) at (1,0) {$v_3$};
\node[vertex] (L5) at (2,0) {$v_2$};
\node[vertex] (L2) at (3,0) {$v_5$};
\node[vertex] (L4) at (4,0) {$v_4$};

\foreach \x in {0,1,2,3,4}
{
    \draw[dotted] (\x,-0.35)--(\x,-0.8);
}


\draw[very thick,->] (5.5,0)--(6.5,0);


\node at (11,1.5) {\textbf{Graph}};

\node[vertex] (v1) at (9, 1.5) {$v_1$};
\node[vertex] (v5) at (7.5,0) {$v_5$};
\node[vertex] (v4) at (10.5,0) {$v_4$};
\node[vertex] (v3) at (9,-1.5) {$v_3$};
\node[vertex] (v2) at (12,0) {$v_2$};


\draw[green!60!black,very thick]
(v3)--node[below left]{$2$}(v5);

\draw[green!60!black,very thick]
(v5)--node[above left]{$2$}(v1);

\draw[green!60!black,very thick]
(v3)--node[below right]{$2$}(v4);

\draw[red,very thick]
(v4)--node[above right]{$1$}(v1);

\draw[green!60!black,very thick]
(v4)--node[above]{$2$}(v2);

\draw[red,very thick]
(v4)--node[above]{$1$}(v5);


\draw[green!60!black,very thick] (6.5,-2.5)--(7,-2.5);
\node[right] at (7.1,-2.5) {$D_c(u,v,f)\ge 2$};

\draw[red,very thick] (10,-2.5)--(10.5,-2.5);
\node[right] at (10.6,-2.5) {$D_c(u,v,f)< 2$};

\end{tikzpicture}

\caption{Illustration of the cyclic antibandwidth constraint on a graph with 5 vertices $\{v_1, \ldots, v_5\}$ for $k = 2$.}
\label{fig:cab_on_graph_5_vertices}
\end{figure}
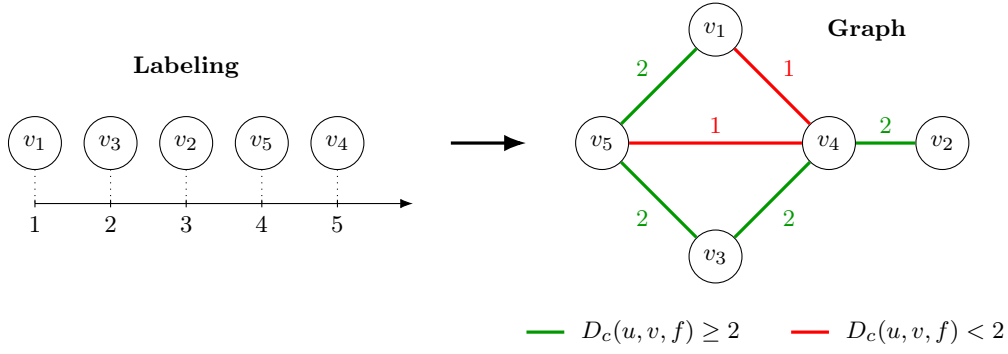

\subsection{Minimum Span Antibandwidth and Cyclic Antibandwidth Labeling}

In this work, we consider the complementary optimization problems of the Antibandwidth and Cyclic Antibandwidth problems. Instead of maximizing the minimum (cyclic) distance between the labels of adjacent vertices, the required minimum (cyclic) distance \(k\) is fixed, and the objective is to determine a feasible labeling with minimum label span.

For the cyclic distance constraint, labels are required to be assigned from the domain
\(\{1,\ldots,\lambda\}\), where \(\lambda\) denotes the largest assigned label (see Appendix \ref{sec:append-label-domain-cab}). To provide a unified formulation, the same label domain is used for the distance constraint. Hence, throughout this paper, we consider labeling $f:V\rightarrow\{1,\ldots,\lambda\}$, rather than $f:V\rightarrow\{0,\ldots,\lambda\}$.
Under this convention, $\operatorname{span}(f)=\lambda-1$.

The two problem variants considered in this paper are defined as follows.

\paragraph{Minimum Span Antibandwidth Labeling (MSABL).}
Given a graph \(G=(V,E)\) and a positive integer \(k\), determine a labeling
\[
f:V\rightarrow\{1,\ldots,\lambda\}
\]
such that
\[
D(u,v,f)\ge k,\qquad \forall \{u,v\}\in E,
\]
while minimizing \(\operatorname{span}(f)\).

\paragraph{Minimum Span Cyclic Antibandwidth Labeling (MSCABL).}
Given a graph \(G=(V,E)\) and a positive integer \(k\), determine a labeling
\[
f:V\rightarrow\{1,\ldots,\lambda\}
\]
such that
\[
D_c(u,v,f)\ge k,\qquad \forall \{u,v\}\in E,
\]
while minimizing \(\operatorname{span}(f)\).

\subsection{Integer Linear Programming Model for SAT Solving}

Inspired by the iterative formulations for the Antibandwidth  \cite{sinnl2021note,truong2025sequential} and Cyclic Antibandwidth \cite{xuan2026solving} problems, we formulate the Minimum Span Labeling problem as a sequence of decision problems. Given an undirected graph $G=(V,E)$, a required minimum (cyclic) distance $k$, and a candidate largest label $\lambda$, the decision problem asks whether there exists a labeling
\[
f:V\rightarrow\{1,\ldots,\lambda\}
\]
that satisfies the (cyclic) distance constraint.

\subsubsection{ILP Model for the Minimum Span Antibandwidth Labeling}

Let $x_{v,l}$ be a binary decision variable such that
\[
x_{v,l}=
\begin{cases}
1, & \text{if vertex }v\text{ is assigned label }l,\\
0, & \text{otherwise}.
\end{cases}
\]

Constraint \eqref{eq:labels} ensures that each vertex is assigned exactly one label.

\begin{align}
\sum_{l = 1}^{\lambda}x_{v,l} = 1,
\qquad
\forall v\in V
\tag{LABELS}
\label{eq:labels}
\end{align}

The distance constraint is enforced by \eqref{eq:distance-k}. For every edge $\{u,v\}$ and every interval of $k$ consecutive labels, at most one of the two vertices may receive a label from that interval. Consequently, the labels assigned to adjacent vertices differ by at least $k$.

\begin{align}
\bigwedge_{L = 1}^{\lambda - k + 1}\sum_{l=L}^{L+k - 1}(x_{u,l} + x_{v,l})\leq 1, \qquad \forall \{u,v\} \in E \tag{DIST-k} \label{eq:distance-k} 
\end{align}

Finally, the assignment of the smallest and largest labels can be enforced by \eqref{eq:min-label} and \eqref{eq:max-label}, which require labels $1$ and $\lambda$, respectively, to be assigned to at least one vertex.

\begin{align}
\sum_{v\in V}x_{v,1}\ge1
\tag{MIN-LABEL}
\label{eq:min-label}
\end{align}

\begin{align}
\sum_{v \in V} x_{v,\lambda}\ge1
\tag{MAX-LABEL}
\label{eq:max-label}
\end{align}

However, since the distance constraint depends only on the difference between labels, any feasible labeling that does not use label $\lambda$ is simply a feasible labeling with a smaller span. 
Therefore, \eqref{eq:max-label} may be omitted without affecting the correctness of the model.

\subsubsection{ILP Model for the Minimum Span Cyclic Antibandwidth Labeling}

The MSCABL formulation shares \eqref{eq:labels}, \eqref{eq:min-label}, and \eqref{eq:max-label} with the MSABL formulation. The only difference lies in the distance constraint \eqref{eq:distance-k}, which is replaced by the cyclic distance constraint \eqref{eq:cyclic-distance-k}.

\begin{equation}
\bigwedge_{L=1}^{\lambda} 
    \sum_{l=L}^{L +k-1} (x_{u,t} + x_{v,t} ) \leq 1,
\qquad 
\left\{
\begin{aligned}
& \forall \{u,v\} \in E \\ 
& t =(l - 1) \bmod \lambda + 1
\end{aligned}
\right.
\tag{CYC-DIST-k}\label{eq:cyclic-distance-k}
\end{equation}

While \eqref{eq:max-label} is optional for the MSABL, it is essential for the MSCABL. 
This is because the cyclic distance is defined with respect to the entire label domain $\{1,\ldots,\lambda\}$.
If the largest label $\lambda$ is not assigned to any vertex, the actual span of the labeling may be smaller, and the cyclic distances are evaluated over an incorrect label domain.
As a result, the computed cyclic distances may be larger than their true values, causing an infeasible labeling to be incorrectly classified as feasible.
Therefore, \eqref{eq:max-label} is required to ensure that the considered span matches the actual span of the labeling and that cyclic distances are evaluated over the correct label domain.

For example, suppose the label domain is $\{1,\ldots,24\}$, the required cyclic distance value is $6$, and the largest label actually used in the labeling is $18$. Consider an edge $\{u,v\}$ whose $u$ and $v$ are assigned labels $1$ and $18$, respectively. With respect to the label domain $\{1,\ldots,24\}$, its cyclic distance is calculated by
\[
\min\bigl(|1-18|,\;24-|1-18|\bigr)=7\ge6,
\]
so $\{u,v\}$ satisfies the cyclic distance constraint. However, since the largest label used is only $18$, the actual cyclic distance is only
\[
\min\bigl(|1-18|,\;18-|1-18|\bigr)=1<6,
\]
which violates the cyclic distance constraint. This example illustrates that, without \eqref{eq:max-label}, the formulation may incorrectly accept an infeasible labeling by evaluating cyclic distances over a larger label domain than the one actually induced by the labeling.

\section{SAT Solving for Minimum Span Labeling}
\label{sec:sat-solving-for-msl}

\subsection{Parallel SAT Solving for MSABL and MSCABL}
\label{sec:parallel-sat-solving}

Algorithm~\ref{alg:parallel_sat} describes our parallel SAT algorithm for solving both MSABL and MSCABL, which solves decision problems concurrently using multiple processors (min. 1). The algorithm is based on the monotonicity property of the decision problem (see Appendix~\ref{appendix-monotonicity-property}). Specifically, if a feasible labeling exists for some $\lambda$, then it also exists for every $\lambda' \ge \lambda$. Conversely, if the problem is infeasible for a given $\lambda$, then it is also infeasible for every $\lambda'' \le \lambda$. Therefore, the optimal span can be found by gradually reducing the range of candidate spans.

\begin{algorithm}
\caption{Parallel SAT solving for MSABL and MSCABL}
\label{alg:parallel_sat}
\begin{algorithmic}[1]
\Require Graph $G=(V,E)$, (cyclic) antibandwidth value $k$, span bounds $(LB,UB)$, SAT solver $\mathcal{S}$, processors $\mathcal{P}$
\Ensure Optimal span and corresponding labeling

\State $\mathcal{Q}\gets\textsc{BuildBFSQueue}(LB,UB)$
\State $span_{invalid}\gets LB-1$
\State $span_{valid}\gets UB+1$

\While{free processor exists \textbf{and} $\textsc{NextSpan}(\mathcal{Q},span_{invalid},span_{valid})\neq\emptyset$}
    \State $span\gets\textsc{NextSpan}(\mathcal{Q},span_{invalid},span_{valid})$
    \State $pid\gets\textsc{LaunchSolver}(G,k,span,\mathcal{S})$
    \State $running[pid]\gets span$
\EndWhile

\While{$running\neq\emptyset$}

    \State $(pid,result)\gets\textsc{WaitAnyPIDComplete}()$
    \State $span\gets running[pid]$
    \State Remove $pid$ from $running$

    \If{$result=$ SAT}
        \State $span_{valid}\gets span$
        \State \textsc{CancelGreater}$(running,span_{valid})$
    \Else
        \State $span_{invalid}\gets span$
        \State \textsc{CancelSmaller}$(running,span_{invalid})$
    \EndIf

    \While{free processor exists \textbf{and} $\textsc{NextSpan}(\mathcal{Q},span_{invalid},span_{valid})\neq\emptyset$}
        \State $span\gets\textsc{NextSpan}(\mathcal{Q},span_{invalid},span_{valid})$
        \State $pid\gets\textsc{LaunchSolver}(G,k,span,\mathcal{S})$
        \State $running[pid]\gets span$
    \EndWhile

\EndWhile

\State \Return $(span_{valid},\text{labeling})$

\end{algorithmic}
\end{algorithm}

Initially, the search interval is defined by the lower and upper bounds $(LB,UB)$. The algorithm maintains two variables during the search, namely $span_{invalid}$ and $span_{valid}$. Variable $span_{invalid}$ stores the largest span that has been proven infeasible, while $span_{valid}$ stores the smallest span that has been proven feasible. Hence, the optimal span always lies between these two values.

Candidate spans are stored in a queue $\mathcal{Q}$ generated by a breadth-first traversal of the binary search tree over the interval $[LB,UB]$. At the beginning, each available processor is assigned one candidate span from $\mathcal{Q}$. Whenever a processor finds the feasibility of its assigned span, the search interval is updated according to the result. If the candidate span is feasible, $span_{valid}$ is updated, and all running processes checking larger spans are terminated because they cannot produce a better solution. Otherwise, $span_{invalid}$ is updated, and all running processes checking smaller spans are terminated because they are also infeasible. The processors that become available are then assigned new candidate spans from $\mathcal{Q}$.

The algorithm stops when there are no unexplored candidate spans between $span_{invalid}$ and $span_{valid}$. Therefore, $span_{valid}$ is the minimum feasible span, and the stored labeling is an optimal solution.

\subsection{Incremental SAT Solving for MSABL}

The parallel algorithm presented in Section~\ref{sec:parallel-sat-solving} repeatedly constructs and solves independent SAT instances for different candidate spans. For MSABL, however, a more efficient approach is possible by taking advantage of incremental SAT solving.

The key observation is that when the largest label decreases from $\lambda$ to $\lambda''$ ($\lambda'' < \lambda$), the SAT encoding for MSABL differs only in the set of labels that become unavailable. In particular, reducing from $\lambda$ to $\lambda''$ only requires adding clauses to exclude labels from $\lambda'' + 1$ to $\lambda$, while all previously generated variables and clauses remain unchanged. Therefore, instead of constructing a new SAT instance for every candidate span, we first build the SAT instance for the $\lambda = UB$. This SAT instance is then reused throughout the search while additional clauses are incrementally added as the candidate span decreases.

Algorithm~\ref{alg:incremental_sat} summarizes the incremental solving procedure. The SAT instance for span $UB$ is solved first. If it returns UNSAT, no feasible labeling exists within the given search interval. Otherwise, the algorithm repeatedly decreases the candidate span, with the decreased amount  determined by the actual span obtained in the previous iteration. At each iteration, the clauses corresponding to the excluded labels are added incrementally to the existing SAT instance, and the solver is invoked again without rebuilding the SAT encoding. As soon as the SAT instance becomes UNSAT, the previous span is returned as the optimal span, together with its corresponding labeling. If every candidate span down to $LB$ is feasible, then $LB$ is the optimal span.

\begin{algorithm}
\caption{Incremental SAT solving for MSABL}
\label{alg:incremental_sat}
\begin{algorithmic}[1]
\Require Graph $G=(V,E)$, antibandwidth value $k$, span bounds $(LB, UB)$, SAT solver $\mathcal{S}$
\Ensure Optimal span and corresponding labeling
\State $result \gets \textsc{Solve}(G,k,UB,\mathcal{S})$
\If{$result =$ UNSAT}
    \State \Return (UB  + 1, labeling)
\EndIf
\State $span \gets GetSpan(\text{labeling})$
\While{$span \geq LB$}
    \State $\mathcal{L} \gets \textsc{GetIgnoredLabel}(span)$
    \State $result \gets \textsc{IncrementalSolve}(\mathcal{S}, \mathcal{L})$
    \If{$result =$ UNSAT}
        \State \Return $(span+1,\text{labeling})$
    \EndIf
    \State $span \gets GetActualSpan(\text{labeling}) - 1$
\EndWhile
\State \Return $(LB,\text{labeling})$
\end{algorithmic}
\end{algorithm}

Compared with solving each decision problem independently, the incremental approach avoids rebuilding the SAT encoding and allows the solver to reuse learned clauses accumulated during previous iterations. This significantly reduces the solving overhead when consecutive candidate spans differ only by a small number of additional constraints.

\subsection{Symmetry Breaking Constraint}

Proposition~\ref{prop:symmetry-labeling} shows that every feasible labeling for MSABL/MSCABL belongs to a pair of symmetric solutions. Consequently, a solver may explore two equivalent branches that lead to the same span.

\begin{proposition}
\label{prop:symmetry-labeling}
For every feasible labeling
\[
f : V \rightarrow \{1,\ldots,\lambda\},
\]
the symmetric labeling
\[
f'(v)=\lambda+1-f(v)
\]
is also feasible.
\end{proposition}

\begin{proof}
First, since 
\[1\le f(v)\le\lambda,\]
it follows that
\[
1\le \lambda+1-f(v)\le\lambda,
\]
thus $f'$ assigns every vertex a label in the domain
$\{1,\ldots,\lambda\}$ and satisfies \eqref{eq:labels}.

Moreover,
\[
f(v)=1 \iff f'(v)=\lambda \qquad \text{and} \qquad
f(v)=\lambda \iff f'(v)=1
\]
implies that if $f$ satisfies
\eqref{eq:min-label} and \eqref{eq:max-label},
then so does $f'$.

For every edge $\{u,v\}\in E$,
\[
|f'(u)-f'(v)| = |(\lambda+1-f(u))-(\lambda+1-f(v))| = |f(u)-f(v)|.
\]
Therefore, \eqref{eq:distance-k} is preserved.

Likewise,
\[
\lambda-|f'(u)-f'(v)|=\lambda-|f(u)-f(v)|,
\]
which implies
\[
\min\!\left(|f'(u)-f'(v)|,\,\lambda-|f'(u)-f'(v)|\right)=\min\!\left(|f(u)-f(v)|,\,\lambda-|f(u)-f(v)|\right).
\]
Hence, \eqref{eq:cyclic-distance-k} is also preserved.

Therefore, $f'$ is feasible whenever $f$ is feasible.
\end{proof}

To eliminate symmetric solutions, we incorporate \eqref{eq:symmetry-breaking}. Specifically, let $v$ be the vertex with the highest degree (and smallest index if tied) in graph $G = (V, E)$. \eqref{eq:symmetry-breaking} restricts that $v$ cannot be assigned any label greater than $\lambda/2$,

\begin{align}
    \bigwedge_{l = \lfloor \lambda/2 \rfloor + 1}^{\lambda} x_{v,l} = 0
    \label{eq:symmetry-breaking} \tag{SYM-BREAK}
\end{align}

This constraint preserves correctness because, for every pair of symmetric labelings $\{f,f'\}$, at least one of them assigns vertex $v$ a label no greater than $\lfloor\lambda/2\rfloor$. Therefore, by restricting the label of vertex $v$ to at most $\lfloor\lambda/2\rfloor$, at least one representative of every symmetric pair remains feasible. Consequently, \eqref{eq:symmetry-breaking} reduces the search space without removing any feasible spans or affecting the optimal solution.

Additionally, the highest-degree vertex is chosen because it participates in the highest number of (cyclic) distance constraints. Restricting the labeling of such a vertex typically propagates more variable assignments during solving, making the symmetry breaking constraint more effective in practice.

\subsection{SAT encoding for ILP model}

This section presents the SAT implementation of the proposed ILP model. First, constraints \eqref{eq:distance-k} and \eqref{eq:cyclic-distance-k} are encoded using a block decomposition technique and Sequential Counter encoding \cite{sinz2005towards}, as used by Hieu et al. on the Antibandwidth \cite{truong2025sequential} and Cyclic Antibandwidth problem \cite{xuan2026solving}. Constraint \eqref{eq:labels} is then encoded by reusing expressions taken from (cyclic) distance encoding. The remaining constraints \eqref{eq:min-label}, \eqref{eq:max-label}, and \eqref{eq:symmetry-breaking} are encoded directly from the corresponding ILP formulations.

\subsubsection{SAT encoding for \eqref{eq:distance-k} and \eqref{eq:cyclic-distance-k}}

To encode \eqref{eq:distance-k} and \eqref{eq:cyclic-distance-k}, we employ the encoding approach proposed by Hieu et al.~\cite{truong2025sequential, xuan2026solving}. In particular, for each edge $\{u,v\}$, each AMO constraint in \eqref{eq:distance-k} (resp. \eqref{eq:cyclic-distance-k}) is decomposed according to Proposition~\ref{prop:decompose_AMO_constraint} into two smaller AMO constraints, each involving variables associated with only one of the two vertices $u$ and $v$, together with an OR clause between two At-Most-Zero (AMZ) constraints. The resulting AMO constraints associated with each vertex form two sequence constraints, which are also referred to as Staircase AMO constraints \cite{fazekas2020duplex, truong2025sequential} (resp. Cyclic Ladder constraints \cite{xuan2026solving})\footnotemark.

\footnotetext{Since Staircase AMO and Cyclic Ladder constraints have the same underlying structure, we refer to both as Ladder AMO constraints throughout the remainder of this paper.}

\begin{proposition}
    \label{prop:decompose_AMO_constraint}
    The constraint $x_{1} + x_{2} + \ldots + x_{n} \leq 1$ holds if and only if for all \(i\) such that \(1 \leq i < n\) :
    \[(x_{1} + \ldots + x_{i}\leq 1)\wedge (x_{i+1} + \ldots + x_{n}\leq 1)\]
    \[\wedge (x_{1} + \ldots + x_{i}\leq 0 \vee x_{i+1} + \ldots + x_{n}\leq 0)\]
\end{proposition}

For example, consider an edge $\{u,v\}$ with a maximum label $\lambda$ and a distance value $k$. The \eqref{eq:distance-k} constraint associated with ${u,v}$ is
\[
\bigwedge_{L = 1}^{\lambda - k + 1}\sum_{l=L}^{L+k - 1}(x_{u,l} + x_{v,l})\leq 1.
\]
By Proposition~\ref{prop:decompose_AMO_constraint}, this constraint can be decomposed into
\[
\begin{aligned}
&\bigwedge_{L=1}^{\lambda-k+1}
  \sum_{l=L}^{L+k-1}x_{u,l}\leq 1
\quad\wedge\quad
\bigwedge_{L=1}^{\lambda-k+1}
  \sum_{l=L}^{L+k-1}x_{v,l}\leq 1
\nonumber\\
&\qquad\wedge\quad
\bigwedge_{L=1}^{\lambda-k+1}
\left(
  \sum_{l=L}^{L+k-1}x_{u,l}\leq0
  \vee
  \sum_{l=L}^{L+k-1}x_{v,l}\leq0
\right).
\end{aligned}
\]
Similarly, for a cyclic distance value $k$, the \eqref{eq:cyclic-distance-k} constraint associated with $\{u,v\}$, i.e.,
\[
\bigwedge_{L=1}^{\lambda} \sum_{l=L}^{L +k-1} (x_{u,t} + x_{v,t} ) \leq 1, \qquad t =(l - 1) \bmod \lambda + 1.
\]
can be decomposed into
\[
\begin{aligned}
&\bigwedge_{L=1}^{\lambda}
  \sum_{l=L}^{L+k-1}x_{u,t}\leq 1
\quad\wedge\quad
\bigwedge_{L=1}^{\lambda}
  \sum_{l=L}^{L+k-1}x_{v,t}\leq 1\\
&\qquad\wedge\quad
\bigwedge_{L=1}^{\lambda}
\left(
  \sum_{l=L}^{L+k-1}x_{u,t}\leq0
  \vee
  \sum_{l=L}^{L+k-1}x_{v,t}\leq0
\right).
\end{aligned}
\]

This decomposition has two advantages. First, it decomposes the constraint on the edges into two constraints associated with two individual vertices, allowing the resulting structures to be shared across edges. Second, it reduces the size of each AMO constraint by approximately half, since each decomposed constraint contains variables associated with only one vertex of the edge. Moreover, the OR constraints between two AMZ constraints can be encoded by exploiting the auxiliary variables introduced when encoding the corresponding AMO constraints~\cite{fazekas2020duplex, truong2025sequential, xuan2026solving}.

Rather than encoding each AMO constraint in a Ladder AMO constraint independently, we first partition its sequence of variables into several windows of width $k$. This partition divides each AMO constraint into smaller expressions. By grouping nested expressions induced by these windows, we obtain a set of blocks in which every expression can be constructed incrementally from smaller expressions in the same block.

For example, Figure~\ref{fig:decompose-ladder-distance-10-4} illustrates the block construction of a Ladder AMO constraint obtained from the decomposition of \eqref{eq:distance-k} with $\lambda = 10$ and $k=4$. The sequence of variables of this Ladder AMO constraint is partitioned into three windows, namely $\{x_{u,1}, \ldots, x_{u,4}\}$, $\{x_{u,5}, \ldots, x_{u,8}\}$, and $\{x_{u,9}, x_{u,10}\}$. Grouping the nested expressions induced by these windows results in four distinct blocks. Similarly, Figure~\ref{fig:decompose-ladder-cyclic-distance-10-4} illustrates the block construction for a Ladder AMO constraint obtained from the decomposition of \eqref{eq:cyclic-distance-k} with $\lambda = 10$ and $k = 4$. In this case, the cyclic structure requires four windows and results in six distinct blocks.

\begin{figure*}
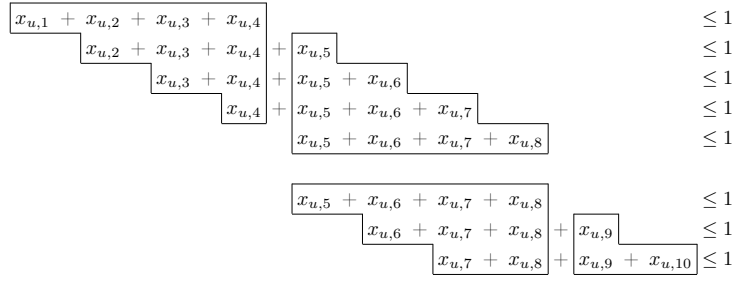

\centering
\renewcommand{\arraystretch}{1.25}
\setlength{\tabcolsep}{2.5pt}
\resizebox{0.75\textwidth}{!}{
\begin{tabular}{llllllllllllllllllll}
\cline{1-7}
\multicolumn{1}{|l}{$x_{u,1}$} & + & $x_{u,2}$ & + & $x_{u,3}$ & + & \multicolumn{1}{l|}{$x_{u,4}$} &  &  &  &  &  &  &  &  &  &  &  &  & $\leq 1$ \\ \cline{1-2} \cline{9-9}
 & \multicolumn{1}{l|}{} & $x_{u,2}$ & + & $x_{u,3}$ & + & \multicolumn{1}{l|}{$x_{u,4}$} & \multicolumn{1}{l|}{+} & \multicolumn{1}{l|}{$x_{u,5}$} &  &  &  &  &  &  &  &  &  &  & $\leq 1$ \\ \cline{3-4} \cline{10-11}
 &  &  & \multicolumn{1}{l|}{} & $x_{u,3}$ & + & \multicolumn{1}{l|}{$x_{u,4}$} & \multicolumn{1}{l|}{+} & $x_{u,5}$ & + & \multicolumn{1}{l|}{$x_{u,6}$} &  &  &  &  &  &  &  &  & $\leq 1$ \\ \cline{5-6} \cline{12-13}
 &  &  &  &  & \multicolumn{1}{l|}{} & \multicolumn{1}{l|}{$x_{u,4}$} & \multicolumn{1}{l|}{+} & $x_{u,5}$ & + & $x_{u,6}$ & + & \multicolumn{1}{l|}{$x_{u,7}$} &  &  &  &  &  &  & $\leq 1$ \\ \cline{7-7} \cline{14-15}
 &  &  &  &  &  &  & \multicolumn{1}{l|}{} & $x_{u,5}$ & + & $x_{u,6}$ & + & $x_{u,7}$ & + & \multicolumn{1}{l|}{$x_{u,8}$} &  &  &  &  & $\leq 1$ \\ \cline{9-15}
 &  &  &  &  &  &  &  &  &  &  &  &  &  &  &  &  &  &  &  \\ \cline{9-15}
 &  &  &  &  &  &  & \multicolumn{1}{l|}{} & $x_{u,5}$ & + & $x_{u,6}$ & + & $x_{u,7}$ & + & \multicolumn{1}{l|}{$x_{u,8}$} &  &  &  &  & $\leq 1$ \\ \cline{9-10} \cline{17-17}
 &  &  &  &  &  &  &  &  & \multicolumn{1}{l|}{} & $x_{u,6}$ & + & $x_{u,7}$ & + & \multicolumn{1}{l|}{$x_{u,8}$} & \multicolumn{1}{l|}{+} & \multicolumn{1}{l|}{$x_{u,9}$} &  &  & $\leq 1$ \\ \cline{11-12} \cline{18-19}
 &  &  &  &  &  &  &  &  &  &  & \multicolumn{1}{l|}{} & $x_{u,7}$ & + & \multicolumn{1}{l|}{$x_{u,8}$} & \multicolumn{1}{l|}{+} & $x_{u,9}$ & + & \multicolumn{1}{l|}{$x_{u,10}$} & $\leq 1$ \\ \cline{13-15} \cline{17-19}
\end{tabular}
}
\caption{Ladder AMO constraint obtained from decomposition of \eqref{eq:distance-k} with $\lambda = 10$ and $k = 4$.}
\label{fig:decompose-ladder-distance-10-4}
\end{figure*}

\begin{figure*}[ht!]
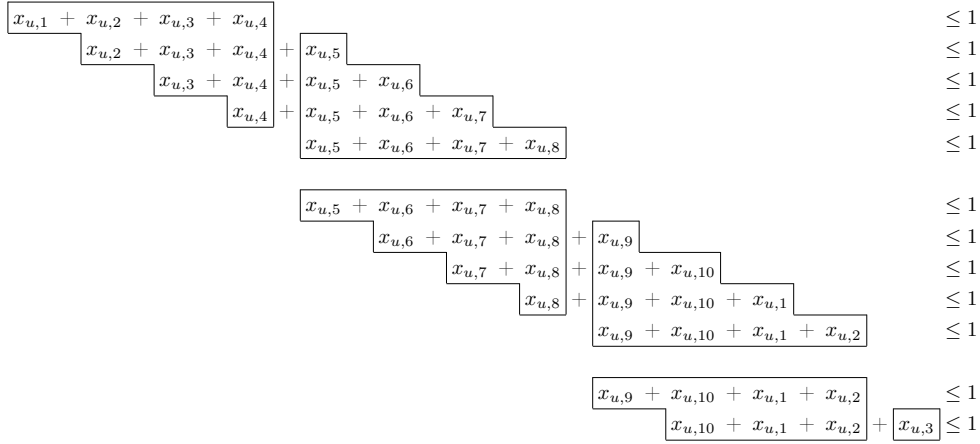

\centering
\renewcommand{\arraystretch}{1.25}
\setlength{\tabcolsep}{2.5pt}
\resizebox{\textwidth}{!}{
\begin{tabular}{llllllllllllllllllllllllll}
\cline{1-7}
\multicolumn{1}{|l}{$x_{u,1}$} & + & $x_{u,2}$ & + & $x_{u,3}$ & + & \multicolumn{1}{l|}{$x_{u,4}$} &  &  &  &  &  &  &  &  &  &  &  &  &  &  &  &  &  &  & $\leq 1$ \\ \cline{1-2} \cline{9-9}
 & \multicolumn{1}{l|}{} & $x_{u,2}$ & + & $x_{u,3}$ & + & \multicolumn{1}{l|}{$x_{u,4}$} & \multicolumn{1}{l|}{+} & \multicolumn{1}{l|}{$x_{u,5}$} &  &  &  &  &  &  &  &  &  &  &  &  &  &  &  &  & $\leq 1$ \\ \cline{3-4} \cline{10-11}
 &  &  & \multicolumn{1}{l|}{} & $x_{u,3}$ & + & \multicolumn{1}{l|}{$x_{u,4}$} & \multicolumn{1}{l|}{+} & $x_{u,5}$ & + & \multicolumn{1}{l|}{$x_{u,6}$} &  &  &  &  &  &  &  &  &  &  &  &  &  &  & $\leq 1$ \\ \cline{5-6} \cline{12-13}
 &  &  &  &  & \multicolumn{1}{l|}{} & \multicolumn{1}{l|}{$x_{u,4}$} & \multicolumn{1}{l|}{+} & $x_{u,5}$ & + & $x_{u,6}$ & + & \multicolumn{1}{l|}{$x_{u,7}$} &  &  &  &  &  &  &  &  &  &  &  &  & $\leq 1$ \\ \cline{7-7} \cline{14-15}
 &  &  &  &  &  &  & \multicolumn{1}{l|}{} & $x_{u,5}$ & + & $x_{u,6}$ & + & $x_{u,7}$ & + & \multicolumn{1}{l|}{$x_{u,8}$} &  &  &  &  &  &  &  &  &  &  & $\leq 1$ \\ \cline{9-15}
 &  &  &  &  &  &  &  &  &  &  &  &  &  &  &  &  &  &  &  &  &  &  &  &  &  \\ \cline{9-15}
 &  &  &  &  &  &  & \multicolumn{1}{l|}{} & $x_{u,5}$ & + & $x_{u,6}$ & + & $x_{u,7}$ & + & \multicolumn{1}{l|}{$x_{u,8}$} &  &  &  &  &  &  &  &  &  &  & $\leq 1$ \\ \cline{9-10} \cline{17-17}
 &  &  &  &  &  &  &  &  & \multicolumn{1}{l|}{} & $x_{u,6}$ & + & $x_{u,7}$ & + & \multicolumn{1}{l|}{$x_{u,8}$} & \multicolumn{1}{l|}{+} & \multicolumn{1}{l|}{$x_{u,9}$} &  &  &  &  &  &  &  &  & $\leq 1$ \\ \cline{11-12} \cline{18-19}
 &  &  &  &  &  &  &  &  &  &  & \multicolumn{1}{l|}{} & $x_{u,7}$ & + & \multicolumn{1}{l|}{$x_{u,8}$} & \multicolumn{1}{l|}{+} & $x_{u,9}$ & + & \multicolumn{1}{l|}{$x_{u,10}$} &  &  &  &  &  &  & $\leq 1$ \\ \cline{13-14} \cline{20-21}
 &  &  &  &  &  &  &  &  &  &  &  &  & \multicolumn{1}{l|}{} & \multicolumn{1}{l|}{$x_{u,8}$} & \multicolumn{1}{l|}{+} & $x_{u,9}$ & + & $x_{u,10}$ & + & \multicolumn{1}{l|}{$x_{u,1}$} &  &  &  &  & $\leq 1$ \\ \cline{15-15} \cline{22-23}
 &  &  &  &  &  &  &  &  &  &  &  &  &  &  & \multicolumn{1}{l|}{} & $x_{u,9}$ & + & $x_{u,10}$ & + & $x_{u,1}$ & + & \multicolumn{1}{l|}{$x_{u,2}$} &  &  & $\leq 1$ \\ \cline{17-23}
 &  &  &  &  &  &  &  &  &  &  &  &  &  &  &  &  &  &  &  &  &  &  &  &  &  \\ \cline{17-23}
 &  &  &  &  &  &  &  &  &  &  &  &  &  &  & \multicolumn{1}{l|}{} & $x_{u,9}$ & + & $x_{u,10}$ & + & $x_{u,1}$ & + & \multicolumn{1}{l|}{$x_{u,2}$} &  &  & $\leq 1$ \\ \cline{17-18} \cline{25-25}
 &  &  &  &  &  &  &  &  &  &  &  &  &  &  &  &  & \multicolumn{1}{l|}{} & $x_{u,10}$ & + & $x_{u,1}$ & + & \multicolumn{1}{l|}{$x_{u,2}$} & \multicolumn{1}{l|}{+} & \multicolumn{1}{l|}{$x_{u,3}$} & $\leq 1$ \\ \cline{19-23} \cline{25-25}
\end{tabular}
}
\caption{Ladder AMO constraint obtained from decomposition of \eqref{eq:cyclic-distance-k} with $\lambda = 10$ and $k = 4$.}
\label{fig:decompose-ladder-cyclic-distance-10-4}
\end{figure*}

Since every expression can be constructed incrementally from a smaller expression in the same block, each block can be encoded using a single Sequential Counter encoding~\cite{sinz2005towards}. Specifically, let $x_{i,j}$ denote the $j^{th}$ variable in the construction order of the expressions in block $i$, and $R_{i,j}$ denote an auxiliary variable that is $true$ if and only if at least one of the first $j$ variables in block $i$ is $true$. The Sequential Counter encoding for block $i$ consists of the four sets of constraints given in \eqref{eq:block_encoding_1}-\eqref{eq:block_encoding_4}. The first three sets of constraints, \eqref{eq:block_encoding_1}-\eqref{eq:block_encoding_3}, ensure that each auxiliary variable $R_{i,j}$ correctly represents whether at least one of the first $j$ variables in block $i$ is $true$. The fourth set of constraints, \eqref{eq:block_encoding_4}, enforces the AMO condition by preventing $x_{i,j}$ from being $true$ when a previously constructed variable in the same block is already $true$.

\begin{minipage}[t]{0.4\textwidth}
\begin{equation}
    \label{eq:block_encoding_1}
    \bigwedge_{j=2}^{k}x_{i,j}\to R_{i,j}
\end{equation}

\begin{equation}
    \label{eq:block_encoding_2}
    \bigwedge_{j=2}^{k} R_{i,j-1} \to R_{i,j}
\end{equation}
\end{minipage}
\hfill
\begin{minipage}[t]{0.4\textwidth}
\begin{equation}
    \label{eq:block_encoding_3}
    \bigwedge_{j=2}^{k}\neg x_{i,j}\wedge\neg R_{i,j-1}\to \neg R_{i,j}
\end{equation}

\begin{equation}
    \label{eq:block_encoding_4}
    \bigwedge_{j=2}^{k} x_{i,j} \to \neg R_{i,j-1}
\end{equation}
\end{minipage}

After encoding all blocks, we connect adjacent blocks to reconstruct the original Ladder AMO constraint. This connection can again be derived from Proposition~\ref{prop:decompose_AMO_constraint}. For example, the constraint
\[x_{u,3} + x_{u,4} + x_{u,5} + x_{u,6} \leq 1\]
can be decomposed as
\[x_{u,3} + x_{u,4} \leq 1 \wedge  x_{u,5} + x_{u,6} \leq 1\]
\[\wedge (x_{u,3} + x_{u,4} \leq 0 \vee  x_{u,5} + x_{u,6} \leq 0).\]
The first two AMO constraints are already enforced by the encoding of the corresponding blocks, while the AMZ constraints in the remaining OR clause can be represented by the negation of the corresponding auxiliary variable $R_{i,j}$, i.e., $\neg R_{i,j}$.

\subsubsection{SAT Encoding for \eqref{eq:labels}}

Constraint~\eqref{eq:labels} can be encoded directly as an Exactly-One (EO) constraint over all label variables of a vertex. However, such an encoding ignores the auxiliary variables and expressions that have already been introduced when encoding \eqref{eq:distance-k} and \eqref{eq:cyclic-distance-k}. To obtain a more compact CNF formula, we reuse these existing expressions to encode \eqref{eq:labels}.

For example, let $\lambda=10$ and $k = 4$. Constraint~\eqref{eq:labels} for a vertex $u$ requires
\[
x_{u,1}+x_{u,2}+\cdots+x_{u,10}=1.
\]
Instead of encoding this EO constraint directly, we partition the variables into three consecutive groups,
\[
x_{u,1}+x_{u,2}+x_{u,3}+x_{u,4}=1,
\]
\[
x_{u,5}+x_{u,6}+x_{u,7}+x_{u,8}=1,
\]
and
\[
x_{u,9}+x_{u,10}=1.
\]
The expressions corresponding to these consecutive groups have already been encoded during the construction of the Ladder AMO constraints for \eqref{eq:distance-k} and \eqref{eq:cyclic-distance-k}, as illustrated in Figure~\ref{fig:decompose-ladder-distance-10-4} and \ref{fig:decompose-ladder-cyclic-distance-10-4}. Consequently, the EO constraint can be encoded by reusing the existing auxiliary variables instead of introducing a new encoding over all label variables.

\subsubsection{SAT Encoding for the other formulations}

The remaining two formulations in the ILP model, including \eqref{eq:min-label} and \eqref{eq:max-label}, and constraints \eqref{eq:symmetry-breaking} are encoded directly into SAT as follows.

\begin{align}
    \sum_{v\in V}x_{v,1}\ge1 \qquad \rightarrow \qquad \bigvee_{v\in V}x_{v,1}
    \tag{MIN-LABEL}
\end{align}

\begin{align}
    \sum_{v\in V}x_{v,\lambda}\ge1 \qquad \rightarrow \qquad \bigvee_{v\in V}x_{v,\lambda}
    \tag{MAX-LABEL}
\end{align}

\begin{align}
   \bigwedge_{l = \lfloor \lambda/2 \rfloor + 1}^{\lambda} x_{v,l} = 0 \qquad \rightarrow \qquad \bigwedge_{l = \lfloor \lambda/2 \rfloor + 1}^{\lambda} \neg x_{v,l}
    \tag{SYM-BREAK}
\end{align}

\section{Experimental Setup}
\label{sec:experimental-setup}

\subsection{Implementation Details}

We implemented a SAT-based framework, referred to as SAT-MSL, for solving the Minimum Span Antibandwidth Labeling (MSABL) and Minimum Span Cyclic Antibandwidth Labeling (MSCABL) problems. The framework is developed in C++ and uses CaDiCaL \cite{biere2024cadical}, version 2.2.1\footnotemark, as the underlying SAT solver.
\footnotetext{https://github.com/arminbiere/cadical/releases/tag/rel-2.2.1}
To control the size of the generated CNF formulas, clauses containing more than 10 literals are decomposed into clauses of at most 10 literals using Tseitin's transformation \cite{tseitin1983complexity}. The framework supports two SAT solving strategies, i.e., parallel SAT solving (Algorithm~\ref{alg:parallel_sat}) and incremental SAT solving (Algorithm~\ref{alg:incremental_sat}). Parallel SAT solving is applicable to both MSABL and MSCABL, whereas incremental SAT solving is employed only for MSABL.

\subsection{Benchmark Dataset}

We evaluate our proposed approaches on a benchmark consisting of 24 graphs selected from the Harwell-Boeing Sparse Matrix Collection~\cite{hb-dataset}. The benchmark contains 12 small graphs and 12 large graphs. Table~\ref{tab:benchmarkData} reports the main information for these instances, including the number of vertices $|V|$, the number of edges $|E|$, and the best-known values obtained from the literature~\cite{lozano2023population, truong2025sequential, xuan2026solving} for the Antibandwidth (ABP) and Cyclic Antibandwidth (CABP) problems, denoted by $ABP_{best}$ and $CABP_{best}$, respectively.

\begin{table*}[ht!]
\centering
\renewcommand{\arraystretch}{1.5}
\caption{Harwell-Boeing Sparse Matrix benchmark.}
\resizebox{\textwidth}{!}{
\begin{tabular}{lccccllcccc}
\multicolumn{5}{c}{\textbf{Small graphs}} &  & \multicolumn{5}{c}{\textbf{Large graphs}}  \\ \cline{1-5} \cline{7-11} 
\textbf{Graph} & \textbf{$|$V$|$} & \textbf{$|$E$|$} & \textbf{$ABP_{best}$} & \textbf{$CABP_{best}$} & \qquad & \textbf{Graph} & \textbf{$|$V$|$} & \textbf{$|$E$|$} & \textbf{$ABP_{best}$} & \textbf{$CABP_{best}$}  \\ \cline{1-5} \cline{7-11} 
A-pores\_1      & 30    & 103   & 6     & 6     & \qquad & M-bcsstk06   & 420   & 3720  & 34    & 33    \\ 
B-ibm32         & 32    & 90    & 9     & 8     & \qquad & N-bcsstk07   & 420   & 3720  & 34    & 33    \\ 
C-bcspwr01      & 39    & 46    & 17    & 13    & \qquad & O-impcol\_d  & 425   & 1267  & 120   & 105   \\ 
D-bcsstk01      & 48    & 176   & 9     & 8     & \qquad & P-can\_\_445 & 445   & 1682  & 90    & 87    \\ 
E-bcspwr02      & 49    & 59    & 21    & 16    & \qquad & Q-494\_bus   & 494   & 586   & 227   & 164   \\ 
F-curtis54      & 54    & 124   & 13    & 10    & \qquad & R-dwt\_\_503 & 503   & 2762  & 63    & 62    \\ 
G-will57        & 57    & 127   & 13    & 11    & \qquad & S-sherman4   & 546   & 1341  & 261   & 258   \\ 
H-impcol\_b     & 59    & 281   & 8     & 7     & \qquad & T-dwt\_\_592 & 592   & 2256  & 113   & 113   \\ 
I-ash85         & 85    & 219   & 23    & 21    & \qquad & U-662\_bus   & 662   & 906   & 220   & 165   \\ 
J-nos4          & 100   & 247   & 35    & 32    & \qquad & V-nos6       & 675   & 1290  & 329   & 328   \\ 
K-dwt\_\_234    & 117   & 162   & 51    & 46    & \qquad & W-685\_bus   & 685   & 1282  & 136   & 114   \\ 
L-bcspwr03      & 118   & 179   & 39    & 29    & \qquad & X-can\_\_715 & 715   & 2975  & 116   & 101   \\ \cline{1-5} \cline{7-11} 
\end{tabular}
}
\label{tab:benchmarkData}
\end{table*}

For each graph, we construct MSABL and MSCABL instances using the corresponding $ABP_{best}$ and $CABP_{best}$ as $k$. Since a feasible ABP (resp. CABP) labeling assigns each vertex a distinct label from $\{1,\ldots,|V|\}$, such a labeling is also a feasible solution for the corresponding MSABL (resp. MSCABL) instance. Consequently, we set the upper bound on the largest label $UB_{\lambda}$ to $|V|$. Therefore, the upper bound of label span is $ UB_{\mathrm{span}} = UB_\lambda - 1 = |V|-1$.

We obtain a simple minimum label span $LB_{span}$ from the distance constraint. Consider an edge $\{u,v\}$ and suppose that $u$ is assigned the lowest label $1$. To satisfy the required distance $k$, the label assigned to $v$ must be at least $k+1$. Hence, the span of any feasible labeling is at least $k$, and we use $LB_{\mathrm{span}} = k$. This bound is used for both MSABL and MSCABL instances.

To analyze the behavior of the proposed methods under different levels of (cyclic) distance tightness, we further extend the benchmark by scaling the (cyclic) distance value, as well as $LB_{span}$ and $UB_{span}$. 
In addition to the original benchmark, corresponding to a coefficient of $1$, we consider four additional coefficients, including $\{0.5,0.75,1.25,1.5\}$. 
For each coefficient, the (cyclic) distance value and the bounds are scaled accordingly. 
Thus, each graph gives rise to five instances for each coefficient variant, resulting in a total of $24\times5=120$ instances for MSABL and another 120 instances for MSCABL. 

Tables~\ref{tab:ab-value-and-bound-of-span-table} and~\ref{tab:cab-value-and-bound-of-span-table} summarize the resulting parameters for MSABL and MSCABL instances, respectively. For each graph and each coefficient, the tables report the scaled (cyclic) distance value together with the corresponding lower and upper bounds on the span.

\begin{table*}[ht!]
\renewcommand{\arraystretch}{1.5}
\setlength{\tabcolsep}{2pt}
\centering
\caption{Distance and bounds of span for each graph with corresponding coefficient.}
\label{tab:ab-value-and-bound-of-span-table}
\resizebox{\textwidth}{!}{
\begin{tabular}{lcccp{0.25cm}cccp{0.25cm}cccp{0.25cm}cccp{0.25cm}ccc}
\hline
\multicolumn{1}{c}{\multirow{2}{*}{\textbf{Graph}}} & \multicolumn{3}{c}{\textbf{1}} & \textbf{} & \multicolumn{3}{c}{\textbf{0.5}} & \textbf{} & \multicolumn{3}{c}{\textbf{0.75}} & \textbf{} & \multicolumn{3}{c}{\textbf{1.25}} & \textbf{} & \multicolumn{3}{c}{\textbf{1.5}} \\ \cline{2-20} 
\multicolumn{1}{c}{} & \textbf{Value} & \textbf{LB\textsubscript{span}} & \textbf{UB\textsubscript{span}} & \textbf{} & \textbf{Value} & \textbf{LB\textsubscript{span}} & \textbf{UB\textsubscript{span}} & \textbf{} & \textbf{Value} & \textbf{LB\textsubscript{span}} & \textbf{UB\textsubscript{span}} & \textbf{} & \textbf{Value} & \textbf{LB\textsubscript{span}} & \textbf{UB\textsubscript{span}} & \textbf{} & \textbf{Value} & \textbf{LB\textsubscript{span}} & \textbf{UB\textsubscript{span}} \\ \hline
A-pores\_1 & 6 & 6 & 29 &  & 3 & 3 & 14 &  & 4 & 4 & 21 &  & 7 & 7 & 37 &  & 9 & 9 & 44 \\
B-ibm32 & 9 & 9 & 31 &  & 4 & 4 & 15 &  & 6 & 6 & 23 &  & 11 & 11 & 39 &  & 13 & 13 & 47 \\
C-bcspwr01 & 17 & 17 & 38 &  & 8 & 8 & 18 &  & 12 & 12 & 28 &  & 21 & 21 & 48 &  & 25 & 25 & 58 \\
D-bcsstk01 & 9 & 9 & 47 &  & 4 & 4 & 23 &  & 6 & 6 & 35 &  & 11 & 11 & 59 &  & 13 & 13 & 71 \\
E-bcspwr02 & 21 & 21 & 48 &  & 10 & 10 & 23 &  & 15 & 15 & 35 &  & 26 & 26 & 61 &  & 31 & 31 & 73 \\
F-curtis54 & 13 & 13 & 53 &  & 6 & 6 & 26 &  & 9 & 9 & 39 &  & 16 & 16 & 67 &  & 19 & 19 & 80 \\
G-will57 & 13 & 13 & 56 &  & 6 & 6 & 27 &  & 9 & 9 & 41 &  & 16 & 16 & 71 &  & 19 & 19 & 85 \\
H-impcol\_b & 8 & 8 & 58 &  & 4 & 4 & 28 &  & 6 & 6 & 43 &  & 10 & 10 & 73 &  & 12 & 12 & 88 \\
I-ash85 & 23 & 23 & 84 &  & 11 & 11 & 41 &  & 17 & 17 & 62 &  & 28 & 28 & 106 &  & 34 & 34 & 127 \\
J-nos4 & 35 & 35 & 99 &  & 17 & 17 & 49 &  & 26 & 26 & 74 &  & 43 & 43 & 124 &  & 52 & 52 & 149 \\
K-dwt\_\_234 & 51 & 51 & 116 &  & 25 & 25 & 57 &  & 38 & 38 & 86 &  & 63 & 63 & 146 &  & 76 & 76 & 175 \\
L-bcspwr03 & 39 & 39 & 117 &  & 19 & 19 & 58 &  & 29 & 29 & 87 &  & 48 & 48 & 147 &  & 58 & 58 & 177 \\
M-bcsstk06 & 34 & 34 & 419 &  & 17 & 17 & 209 &  & 25 & 25 & 314 &  & 42 & 42 & 524 &  & 51 & 51 & 629 \\
N-bcsstk07 & 34 & 34 & 419 &  & 17 & 17 & 209 &  & 25 & 25 & 314 &  & 42 & 42 & 524 &  & 51 & 51 & 629 \\
O-impcol\_d & 120 & 120 & 424 &  & 60 & 60 & 211 &  & 90 & 90 & 317 &  & 150 & 150 & 531 &  & 180 & 180 & 637 \\
P-can\_\_445 & 90 & 90 & 444 &  & 45 & 45 & 221 &  & 67 & 67 & 332 &  & 112 & 112 & 556 &  & 135 & 135 & 667 \\
Q-494\_bus & 227 & 227 & 493 &  & 113 & 113 & 246 &  & 170 & 170 & 369 &  & 283 & 283 & 617 &  & 340 & 340 & 740 \\
R-dwt\_\_503 & 63 & 63 & 502 &  & 31 & 31 & 250 &  & 47 & 47 & 376 &  & 78 & 78 & 628 &  & 94 & 94 & 754 \\
S-sherman4 & 261 & 261 & 545 &  & 130 & 130 & 272 &  & 195 & 195 & 408 &  & 326 & 326 & 682 &  & 391 & 391 & 818 \\
T-dwt\_\_592 & 113 & 113 & 591 &  & 56 & 56 & 295 &  & 84 & 84 & 443 &  & 141 & 141 & 739 &  & 169 & 169 & 887 \\
U-662\_bus & 220 & 220 & 661 &  & 110 & 110 & 330 &  & 165 & 165 & 495 &  & 275 & 275 & 827 &  & 330 & 330 & 992 \\
V-nos6 & 329 & 329 & 674 &  & 164 & 164 & 336 &  & 246 & 246 & 505 &  & 411 & 411 & 843 &  & 493 & 493 & 1012 \\
W-685\_bus & 136 & 136 & 684 &  & 68 & 68 & 341 &  & 102 & 102 & 512 &  & 170 & 170 & 856 &  & 204 & 204 & 1027 \\
X-can\_\_715 & 116 & 116 & 714 &  & 58 & 58 & 356 &  & 87 & 87 & 535 &  & 145 & 145 & 893 &  & 174 & 174 & 1072 \\ \hline
\end{tabular}
}
\end{table*}

\begin{table*}[ht!]
\renewcommand{\arraystretch}{1.5}
\setlength{\tabcolsep}{2pt}
\centering
\caption{Cyclic distance and bounds of span for each graph with corresponding coefficient.}
\label{tab:cab-value-and-bound-of-span-table}
\resizebox{\textwidth}{!}{
\begin{tabular}{lcccp{0.25cm}cccp{0.25cm}cccp{0.25cm}cccp{0.25cm}ccc}
\hline
\multicolumn{1}{c}{\multirow{2}{*}{\textbf{Graph}}} & \multicolumn{3}{c}{\textbf{1}} & \textbf{} & \multicolumn{3}{c}{\textbf{0.5}} & \textbf{} & \multicolumn{3}{c}{\textbf{0.75}} & \textbf{} & \multicolumn{3}{c}{\textbf{1.25}} & \textbf{} & \multicolumn{3}{c}{\textbf{1.5}} \\ \cline{2-20} 
\multicolumn{1}{c}{} & \textbf{Value} & \textbf{LB\textsubscript{span}} & \textbf{UB\textsubscript{span}} & \textbf{} & \textbf{Value} & \textbf{LB\textsubscript{span}} & \textbf{UB\textsubscript{span}} & \textbf{} & \textbf{Value} & \textbf{LB\textsubscript{span}} & \textbf{UB\textsubscript{span}} & \textbf{} & \textbf{Value} & \textbf{LB\textsubscript{span}} & \textbf{UB\textsubscript{span}} & \textbf{} & \textbf{Value} & \textbf{LB\textsubscript{span}} & \textbf{UB\textsubscript{span}} \\ \hline
A-pores\_1 & 6 & 6 & 29 &  & 3 & 3 & 14 &  & 4 & 4 & 21 &  & 7 & 7 & 37 &  & 9 & 9 & 44 \\
B-ibm32 & 8 & 8 & 31 &  & 4 & 4 & 15 &  & 6 & 6 & 23 &  & 10 & 10 & 39 &  & 12 & 12 & 47 \\
C-bcspwr01 & 13 & 13 & 38 &  & 6 & 6 & 18 &  & 9 & 9 & 28 &  & 16 & 16 & 48 &  & 19 & 19 & 58 \\
D-bcsstk01 & 8 & 8 & 47 &  & 4 & 4 & 23 &  & 6 & 6 & 35 &  & 10 & 10 & 59 &  & 12 & 12 & 71 \\
E-bcspwr02 & 16 & 16 & 48 &  & 8 & 8 & 23 &  & 12 & 12 & 35 &  & 20 & 20 & 61 &  & 24 & 24 & 73 \\
F-curtis54 & 10 & 10 & 53 &  & 5 & 5 & 26 &  & 7 & 7 & 39 &  & 12 & 12 & 67 &  & 15 & 15 & 80 \\
G-will57 & 11 & 11 & 56 &  & 5 & 5 & 27 &  & 8 & 8 & 41 &  & 13 & 13 & 71 &  & 16 & 16 & 85 \\
H-impcol\_b & 7 & 7 & 58 &  & 3 & 3 & 28 &  & 5 & 5 & 43 &  & 8 & 8 & 73 &  & 10 & 10 & 88 \\
I-ash85 & 21 & 21 & 84 &  & 10 & 10 & 41 &  & 15 & 15 & 62 &  & 26 & 26 & 106 &  & 31 & 31 & 127 \\
J-nos4 & 32 & 32 & 99 &  & 16 & 16 & 49 &  & 24 & 24 & 74 &  & 40 & 40 & 124 &  & 48 & 48 & 149 \\
K-dwt\_\_234 & 46 & 46 & 116 &  & 23 & 23 & 57 &  & 34 & 34 & 86 &  & 57 & 57 & 146 &  & 69 & 69 & 175 \\
L-bcspwr03 & 29 & 29 & 117 &  & 14 & 14 & 58 &  & 21 & 21 & 87 &  & 36 & 36 & 147 &  & 43 & 43 & 177 \\
M-bcsstk06 & 33 & 33 & 419 &  & 16 & 16 & 209 &  & 24 & 24 & 314 &  & 41 & 41 & 524 &  & 49 & 49 & 629 \\
N-bcsstk07 & 33 & 33 & 419 &  & 16 & 16 & 209 &  & 24 & 24 & 314 &  & 41 & 41 & 524 &  & 49 & 49 & 629 \\
O-impcol\_d & 105 & 105 & 424 &  & 52 & 52 & 211 &  & 78 & 78 & 317 &  & 131 & 131 & 531 &  & 157 & 157 & 637 \\
P-can\_\_445 & 87 & 87 & 444 &  & 43 & 43 & 221 &  & 65 & 65 & 332 &  & 108 & 108 & 556 &  & 130 & 130 & 667 \\
Q-494\_bus & 164 & 164 & 493 &  & 82 & 82 & 246 &  & 123 & 123 & 369 &  & 205 & 205 & 617 &  & 246 & 246 & 740 \\
R-dwt\_\_503 & 62 & 62 & 502 &  & 31 & 31 & 250 &  & 46 & 46 & 376 &  & 77 & 77 & 628 &  & 93 & 93 & 754 \\
S-sherman4 & 258 & 258 & 545 &  & 129 & 129 & 272 &  & 193 & 193 & 408 &  & 322 & 322 & 682 &  & 387 & 387 & 818 \\
T-dwt\_\_592 & 113 & 113 & 591 &  & 56 & 56 & 295 &  & 84 & 84 & 443 &  & 141 & 141 & 739 &  & 169 & 169 & 887 \\
U-662\_bus & 165 & 165 & 661 &  & 82 & 82 & 330 &  & 123 & 123 & 495 &  & 206 & 206 & 827 &  & 247 & 247 & 992 \\
V-nos6 & 328 & 328 & 674 &  & 164 & 164 & 336 &  & 246 & 246 & 505 &  & 410 & 410 & 843 &  & 492 & 492 & 1012 \\
W-685\_bus & 114 & 114 & 684 &  & 57 & 57 & 341 &  & 85 & 85 & 512 &  & 142 & 142 & 856 &  & 171 & 171 & 1027 \\
X-can\_\_715 & 101 & 101 & 714 &  & 50 & 50 & 356 &  & 75 & 75 & 535 &  & 126 & 126 & 893 &  & 151 & 151 & 1072 \\ \hline
\end{tabular}
}
\end{table*}

The scaling factors allow us to evaluate the algorithms beyond the original problem parameters. In particular, coefficients below $1$ produce smaller (cyclic) distance requirements and narrower search intervals, whereas coefficients above $1$ produce more demanding (cyclic) distance requirements and larger search intervals. This provides a controlled way to examine the scalability and robustness of the proposed SAT-based approaches across different problem scales.

To further investigate the effect of restricting the use of labels, we additionally consider a no-hole constraint for the coefficients 0.5, 0.75, and 1. Originally, labels $1$ and $\lambda$ are required to be used, ensuring that the label span is exactly $\lambda-1$. The no-hole constraint further requires that every label between these two boundary labels be assigned to at least one vertex, i.e., no label within the span may remain unused. We do not impose this constraint for coefficients 1.25 and 1.5, since the corresponding scaled (cyclic) distance may require the largest label to exceed the number of vertices, making it impossible to use every label within the span.

\subsection{Compared Methods}

We compare the proposed SAT-based approach with three well-known optimization-based methods, which are CPLEX\textsubscript{CP}, CPLEX\textsubscript{MIP}, and Gurobi. CPLEX\textsubscript{CP} employs Constraint Programming (CP), whereas CPLEX\textsubscript{MIP} and Gurobi are based on Mixed Integer Programming (MIP). In our implementation, CPLEX\textsubscript{CP} and CPLEX\textsubscript{MIP} use the CP Optimizer and CPLEX Optimizer, respectively, provided by IBM ILOG CPLEX Optimization Studio\footnotemark~version 22.2.0.0,
\footnotetext{\url{https://www.ibm.com/products/ilog-cplex-optimization-studio}}
while Gurobi uses the Gurobi Optimizer\footnotemark~version 13.0.2.
\footnotetext{\url{https://www.gurobi.com/}}
All three approaches model the problem directly using the native constraint modeling capabilities of their respective optimization solvers, with the label span minimized as the objective function.

The optimization models used by these methods are defined as follows. Given an undirected graph $G=(V,E)$, a (cyclic) distance value $k$, and lower/upper bounds on the label span, denoted by $LB_{span}$/$UB_{span}$. For each vertex $v\in V$, let $f(v)$ be an integer variable representing the label assigned to $v$, and $\lambda$ be an integer variable representing the largest label used. The constraints defining a feasible labeling are
\[
\renewcommand{\arraystretch}{1.3}
\begin{array}{lll}
    & \forall v\in V:       & 1 \leq f(v) \leq UB_{span} + 1,\\
    & \forall v\in V:       & \lambda = \max f(v),\\
    &                       & LB_{span} + 1 \leq \lambda \leq UB_{span} + 1,\\
    & \exists v\in V:       & f(v)=1.
\end{array}
\]
The first constraint bounds the label assigned to each vertex based on the upper bound on the span. The second defines $\lambda$ as the largest label used, while the third restricts $\lambda$ according to the given span bounds. The last constraint ensures that label $1$ is used. Consequently, the label span is exactly $\lambda-1$.

We also apply symmetry breaking based on the observation that, for a feasible labeling with the largest label $\lambda$, the symmetric labeling $f'(v)=\lambda+1-f(v)$ is also feasible. Let $v^\star$ be a vertex of maximum degree. We restrict the label assigned to this vertex to the lower half of the label domain by imposing
\[
2f(v^\star)\leq UB_{span} + 2.
\]

Using these constraints, the optimization model for the Minimum Span Antibandwidth Labeling problem is
\[
\renewcommand{\arraystretch}{1.3}
\begin{array}{lll}
\min \quad          & \lambda - 1           & \\
\text{s.t.}\quad    & \forall v\in V:       & 1 \leq f(v) \leq UB_{span} + 1,\\
                    & \forall v\in V:       & \lambda = \max f(v),\\
                    &                       & LB_{span} + 1 \leq \lambda \leq UB_{span} + 1,\\
                    & \exists v\in V:       & f(v)=1,\\
                    &                       & 2f(v^\star)\leq UB_{span} + 2,\\
                    & \forall \{u,v\}\in E: & |f(u)-f(v)| \geq k.
\end{array}
\]
For the Minimum Span Cyclic Antibandwidth Labeling problem, the corresponding optimization model is
\[
\renewcommand{\arraystretch}{1.3}
\begin{array}{lll}
\min \quad          & \lambda - 1           & \\
\text{s.t.}\quad    & \forall v\in V:       & 1 \leq f(v) \leq UB_{span} + 1,\\
                    & \forall v\in V:       & \lambda = \max f(v),\\
                    &                       & LB_{span} + 1 \leq \lambda \leq UB_{span} + 1,\\
                    & \exists v\in V:       & f(v)=1,\\
                    &                       & 2f(v^\star)\leq UB_{span} + 2,\\
                    & \forall \{u,v\}\in E: & \min\{|f(u)-f(v)|,\lambda-|f(u)-f(v)|\}\geq k.
\end{array}
\]

\subsection{Experimental Environment and Metrics}

All experiments are conducted on a virtual machine provided by the Google Cloud Platform\footnotemark, running Debian 12 (Bookworm) on a 64-bit architecture. 
\footnotetext{https://console.cloud.google.com}
The machine uses the \texttt{e2-highmem-8} machine type and is equipped with 4 physical cores, 8 virtual CPUs, and 64 GB of memory, with no GPU. Each instance is solved with a time limit of 1800 seconds and a memory limit of 45 GB. CPLEX\textsubscript{CP}, CPLEX\textsubscript{MIP}, and Gurobi are run with their default configurations. Given the hardware configuration described above, these approaches can use up to 8 concurrent processes. For a fair comparison, SAT-MSL with the parallel strategy is also configured to use up to 8 concurrent processes.

We perform three types of comparisons to evaluate the performance of the considered approaches. First, we compare the methods based on three primary performance metrics, including the number of solved instances ($\#Solved$), the number of instances for which an optimal span is obtained ($\#Optimal$), and the number of instances for which a best-known span is obtained ($\#Best$). If two or more approaches obtain the same span, but only some of them prove its optimality, only those that prove optimality are credited with the best-known span. Furthermore, if none of the approaches finds a feasible span within the time limit, the instance is not considered to have a best-known result for any approach.

Second, we use the Friedman test~\cite{friedman1937use} to compare the average rankings ($\#Avg.Rank$) of the considered methods across the test instances. This non-parametric test is used to determine whether there are statistically significant differences among the overall performances of the methods.

Finally, we perform pairwise comparisons between the proposed and competing methods using the Wilcoxon signed-rank test~\cite{garcia2009study}. This test is used to assess whether the performance difference between any two methods is statistically significant. We also note that, in instances where an approach fails to find any feasible span within the time limit, we use (UB + 1) as the approach's result in the Wilcoxon test. Besides, when an approach proves a span to be optimal, we consider (span - 1) as its result, thereby distinguishing a proven optimal result from the same span obtained without a proof of optimality.

Together, these three comparisons provide complementary views of the computational performance of the considered approaches. They not only cover the number of instances successfully solved and the quality of the obtained solutions, but also the statistical significance of the observed performance differences.

\section{Experimental Results}
\label{sec:experimental-results}

This section presents the experimental results of all approaches based on two variants of the proposed optimization model. The first is the base model, which does not impose the no-hole constraint, while the second additionally incorporates the no-hole constraint, requiring every label within the span to be used. The results reported in this section provide an overall summary of the experimental findings, while detailed results for individual instances are provided in the Appendix~\ref{sec:detailed-experiment-results}.

\subsection{Experimental Results on the Base Model}

Table~\ref{tab:experimental-results-summary-base-model} summarizes the experimental results for both MSABL and MSCABL using the base model, i.e., the model without no-hole constraint. The results are reported for five coefficient values, namely 1, 0.5, 0.75, 1.25, and 1.5. For each coefficient, we report the number of solved instances ($\#Solved$), the number of instances for which the optimal span is obtained ($\#Optimal$), the number of instances for which the best-known span is obtained ($\#Best$), and the average rank ($\#Avg.Rank$). We also note that the parallel and incremental SAT solving approaches are denoted by SAT\textsubscript{Par} and SAT\textsubscript{Inc}, respectively.

\begin{table*}[ht!]
\renewcommand{\arraystretch}{1.25}
\setlength{\tabcolsep}{3pt}
\centering
\caption{Experimental results on the base model.}
\label{tab:experimental-results-summary-base-model}
\resizebox{\textwidth}{!}{
\begin{tabular}{clcccccp{0.25cm}cccc}
\hline
\multirow{2}{*}{\textbf{Coef.}} &  & \multicolumn{5}{c}{\textbf{MSABL}} &  & \multicolumn{4}{c}{\textbf{MSCABL}} \\ \cline{3-7} \cline{9-12} 
 &  & \textbf{CPLEX\textsubscript{CP}} & \textbf{CPLEX\textsubscript{MIP}} & \textbf{Gurobi} & \textbf{SAT\textsubscript{Par}} & \textbf{SAT\textsubscript{Inc}} &  & \textbf{CPLEX\textsubscript{CP}} & \textbf{CPLEX\textsubscript{MIP}} & \textbf{Gurobi} & \textbf{SAT\textsubscript{Par}} \\ \hline
\multirow{4}{*}{0.5} & $\#Solved$ & 24 & 16 & 17 & 24 & 24 &  & 24 & 14 & 17 & 24 \\
 & $\#Optimal$ & 18 & 10 & 13 & \textbf{22} & \textbf{22} &  & 18 & 9 & 13 & \textbf{22} \\
 & $\#Best$ & 20 & 10 & 13 & 22 & \textbf{24} &  & 20 & 9 & 13 & \textbf{24} \\
 & $\#Avg.Rank$ & 2.63 & 3.9 & 3.65 & 2.48 & \textbf{2.35} &  & 2.06 & 3.25 & 2.81 & \textbf{1.88} \\ \hline
\multirow{4}{*}{0.75} & $\#Solved$ & 24 & 14 & 16 & 24 & 24 &  & 24 & 13 & 16 & 24 \\
 & $\#Optimal$ & 18 & 11 & 13 & \textbf{22} & \textbf{22} &  & 18 & 10 & 13 & \textbf{22} \\
 & $\#Best$ & 20 & 11 & 13 & \textbf{24} & \textbf{24} &  & 20 & 10 & 13 & \textbf{24} \\
 & $\#Avg.Rank$ & 2.69 & 3.9 & 3.58 & \textbf{2.42} & \textbf{2.42} &  & 2.08 & 3.17 & 2.85 & \textbf{1.9} \\ \hline
\multirow{4}{*}{1} & $\#Solved$ & 24 & 14 & 24 & 24 & 24 &  & 24 & 14 & 17 & 24 \\
 & $\#Optimal$ & 18 & 11 & 17 & \textbf{22} & \textbf{22} &  & 18 & 10 & 11 & \textbf{21} \\
 & $\#Best$ & 20 & 11 & 17 & \textbf{24} & \textbf{24} &  & 20 & 10 & 11 & \textbf{24} \\
 & $\#Avg.Rank$ & 2.83 & 4.08 & 3.08 & \textbf{2.5} & \textbf{2.5} &  & 2.02 & 3.17 & 2.96 & \textbf{1.85} \\ \hline
\multirow{4}{*}{1.25} & $\#Solved$ & 24 & 13 & 24 & 24 & 24 &  & 24 & 13 & 17 & 24 \\
 & $\#Optimal$ & 17 & 11 & 17 & 22 & \textbf{24} &  & 17 & 10 & 12 & \textbf{21} \\
 & $\#Best$ & 17 & 11 & 17 & 22 & \textbf{24} &  & 19 & 10 & 12 & \textbf{22} \\
 & $\#Avg.Rank$ & 2.94 & 4.06 & 2.96 & 2.65 & \textbf{2.4} &  & 2.08 & 3.15 & 2.88 & \textbf{1.9} \\ \hline
\multirow{4}{*}{1.5} & $\#Solved$ & 24 & 6 & 24 & 24 & 24 &  & 24 & 13 & 16 & 24 \\
 & $\#Optimal$ & 17 & 6 & 18 & \textbf{22} & \textbf{22} &  & 17 & 7 & 13 & \textbf{21} \\
 & $\#Best$ & 19 & 6 & 18 & \textbf{24} & \textbf{24} &  & 19 & 7 & 13 & \textbf{22} \\
 & $\#Avg.Rank$ & 2.79 & 4.5 & 2.92 & \textbf{2.4} & \textbf{2.4} &  & 2.02 & 3.33 & 2.79 & \textbf{1.85} \\ \hline
\multirow{4}{*}{Overall} & $\#Solved$ & 120 & 63 & 105 & 120 & 120 & \multicolumn{1}{c}{} & 120 & 67 & 83 & 120 \\
 & $\#Optimal$ & 88 & 49 & 78 & 110 & \textbf{112} & \multicolumn{1}{c}{} & 88 & 46 & 62 & \textbf{107} \\
 & $\#Best$ & 96 & 49 & 78 & 116 & \textbf{120} & \multicolumn{1}{c}{} & 98 & 46 & 62 & \textbf{116} \\
 & $\#Avg.Rank$ & 2.78 & 4.09 & 3.24 & 2.49 & \textbf{2.41} & \multicolumn{1}{c}{} & 2.05 & 3.21 & 2.86 & \textbf{1.88} \\ \hline
\multicolumn{12}{l}{The best result for each criterion is highlighted in bold.}
\end{tabular}
}
\end{table*}

For MSABL, the proposed SAT-based approaches achieve the best overall performance among the evaluated approaches. Both SAT\textsubscript{Par} and SAT\textsubscript{Inc} solve all 120 instances, matching CPLEX\textsubscript{CP}, while Gurobi and CPLEX\textsubscript{MIP} solve 105 and 63 instances, respectively. In terms of solution quality, SAT\textsubscript{Inc} obtains optimal spans for 112 instances and best-known spans for all 120 instances, while SAT\textsubscript{Par} obtains 110 optimal and 116 best-known solutions. By comparison, CPLEX\textsubscript{CP} obtains 88 optimal and 96 best-known solutions, whereas CPLEX\textsubscript{MIP} and Gurobi obtain substantially fewer optimal and best-known solutions. The average ranks also favor the SAT-based approaches. SAT\textsubscript{Inc} achieves the best overall rank of 2.41, followed by SAT\textsubscript{Par} with 2.49. CPLEX\textsubscript{CP}, Gurobi, and CPLEX\textsubscript{MIP} obtain average ranks of 2.78, 3.24, and 4.09, respectively. These results indicate that the SAT-based approaches provide a clear advantage in both solution quality and overall ranking, while also maintaining the highest number of solved instances.

The results across the five coefficient values further demonstrate the consistent performance of the SAT-based approaches. Both SAT\textsubscript{Par} and SAT\textsubscript{Inc} solve all 24 instances for every coefficient value. SAT\textsubscript{Inc} obtains at least 22 optimal solutions for each coefficient and achieves 24 best-known solutions for all five coefficients. It also achieves the best average rank for coefficients 0.5 and 1.25, with ranks of 2.35 and 2.40, respectively, while tying with SAT\textsubscript{Par} for the remaining coefficients. SAT\textsubscript{Par} similarly achieves competitive average ranks ranging from 2.40 to 2.65. The relatively small difference between the two SAT strategies suggests that both parallel and incremental solving are effective for MSABL, with the incremental strategy providing a slight overall advantage in solution quality and ranking.

For MSCABL, SAT\textsubscript{Par} achieves the best overall performance among the evaluated approaches. It solves all 120 instances, obtaining 107 optimal and 116 best-known solutions. CPLEX\textsubscript{CP} also solves all 120 instances, but obtains only 88 optimal and 98 best-known solutions. Gurobi and CPLEX\textsubscript{MIP} solve 83 and 67 instances, respectively, and obtain considerably fewer optimal and best-known solutions. The difference is also reflected in the average ranks. In particular, SAT\textsubscript{Par} achieves the best overall rank of 1.88, followed by CPLEX\textsubscript{CP} with 2.05, Gurobi with 2.86, and CPLEX\textsubscript{MIP} with 3.21. Thus, although CPLEX\textsubscript{CP} is competitive in terms of the number of solved instances, SAT\textsubscript{Par} provides substantially better solution quality and achieves the best overall ranking.

The results for the individual coefficient values further demonstrate the stability of SAT\textsubscript{Par} on MSCABL. It achieves the best average rank for all five coefficients, with ranks ranging from 1.85 to 1.90. It also obtains at least 21 optimal solutions for every coefficient and between 22 and 24 best-known solutions. In contrast, the optimization-based approaches show greater variation in both the number of solved instances and the quality of the obtained solutions, particularly for CPLEX\textsubscript{MIP} and Gurobi. Overall, these results show that the SAT-based approach maintains consistently strong performance across different scaling values for MSCABL.

To determine whether the observed differences in Table~\ref{tab:experimental-results-summary-base-model} are statistically significant, we perform the Wilcoxon signed-rank test \cite{garcia2009study}. The results are reported in Table~\ref{tab:wilcoxon-base-model}, with $p-value$ rounded to four decimal places. The test is conducted pairwise between each SAT-based and optimization-based approaches. For MSABL, we additionally compare SAT\textsubscript{Par} and SAT\textsubscript{Inc}.

\begin{table*}[ht!]
\renewcommand{\arraystretch}{1.25}
\setlength{\tabcolsep}{3pt}
\centering
\caption{Wilcoxon signed-rank test results for the base model.}
\label{tab:wilcoxon-base-model}
\resizebox{\textwidth}{!}{
\begin{tabular}{clccccp{0.25cm}ccccp{0.25cm}ccc}
\hline
\multicolumn{1}{l}{\multirow{3}{*}{Coef.}} & \multirow{3}{*}{} & \multicolumn{9}{c}{\textbf{MSABL}} &  & \multicolumn{3}{c}{\textbf{MSCABL}} \\ \cline{3-11} \cline{13-15} 
\multicolumn{1}{l}{} &  & \multicolumn{4}{c}{\textbf{SAT\textsubscript{Par}}} &  & \multicolumn{4}{c}{\textbf{SAT\textsubscript{Inc}}} &  & \multicolumn{3}{c}{\textbf{SAT\textsubscript{Par}}} \\ \cline{3-6} \cline{8-11} \cline{13-15} 
\multicolumn{1}{l}{} &  & \textbf{CPLEX\textsubscript{CP}} & \textbf{CPLEX\textsubscript{MIP}} & \textbf{Gurobi} & \textbf{SAT\textsubscript{Inc}} &  & \textbf{CPLEX\textsubscript{CP}} & \textbf{CPLEX\textsubscript{MIP}} & \textbf{Gurobi} & \textbf{SAT\textsubscript{Par}} &  & \textbf{CPLEX\textsubscript{CP}} & \textbf{CPLEX\textsubscript{MIP}} & \textbf{Gurobi} \\ \hline
\multirow{3}{*}{0.5} & $R^+$ & 14 & 105 & 66 & 0 &  & 10 & 105 & 66 & 3 &  & 10 & 120 & 66 \\
 & $R^-$ & 7 & 0 & 0 & 3 &  & 0 & 0 & 0 & 0 &  & 0 & 0 & 0 \\
 & $p-value$ & 0.4142 & 0.0009 & 0.0033 & 0.1573 &  & 0.0455 & 0.0009 & 0.0033 & 0.1573 &  & 0.0588 & 0.0006 & 0.0033 \\ \hline
\multirow{3}{*}{0.75} & $R^+$ & 10 & 91 & 66 & 0 &  & 10 & 91 & 66 & 0 &  & 10 & 105 & 66 \\
 & $R^-$ & 0 & 0 & 0 & 0 &  & 0 & 0 & 0 & 0 &  & 0 & 0 & 0 \\
 & $p-value$ & 0.0455 & 0.0015 & 0.0033 & - &  & 0.0455 & 0.0015 & 0.0033 & - &  & 0.0588 & 0.0010 & 0.0033 \\ \hline
\multirow{3}{*}{1} & $R^+$ & 10 & 91 & 28 & 0 &  & 10 & 91 & 28 & 0 &  & 10 & 105 & 91 \\
 & $R^-$ & 0 & 0 & 0 & 0 &  & 0 & 0 & 0 & 0 &  & 0 & 0 & 0 \\
 & $p-value$ & 0.0455 & 0.0015 & 0.0139 & - &  & 0.0455 & 0.0015 & 0.0139 & - &  & 0.0588 & 0.0010 & 0.0014 \\ \hline
\multirow{3}{*}{1.25} & $R^+$ & 15 & 91 & 17 & 0 &  & 28 & 91 & 28 & 3 &  & 17 & 105 & 78 \\
 & $R^-$ & 13 & 0 & 11 & 3 &  & 0 & 0 & 0 & 0 &  & 11 & 0 & 0 \\
 & $p-value$ & 0.8605 & 0.0015 & 0.6049 & 0.1573 &  & 0.0082 & 0.0015 & 0.0114 & 0.1573 &  & 0.6049 & 0.0010 & 0.0020 \\ \hline
\multirow{3}{*}{1.5} & $R^+$ & 15 & 171 & 21 & 0 &  & 15 & 171 & 21 & 0 &  & 17 & 153 & 66 \\
 & $R^-$ & 0 & 0 & 0 & 0 &  & 0 & 0 & 0 & 0 &  & 11 & 0 & 0 \\
 & $p-value$ & 0.0253 & 0.0002 & 0.0256 & - &  & 0.0253 & 0.0002 & 0.0256 & - &  & 0.6049 & 0.0003 & 0.0033 \\ \hline
\multirow{3}{*}{Overall} & $R^+$ & 275 & 2556 & 863 & 0 &  & 300 & 2556 & 903 & 10 &  & 271 & 2775 & 1711 \\
 & $R^-$ & 76 & 0 & 40 & 10 &  & 0 & 0 & 0 & 0 &  & 80 & 0 & 0 \\
 & $p-value$ & 0.0051 & 0.0000 & 0.0000 & 0.0633 &  & 0.0000 & 0.0000 & 0.0000 & 0.0633 &  & 0.0121 & 0.0000 & 0.0000 \\ \hline
\end{tabular}
}
\end{table*}

For MSABL, the overall results show that both SAT-based approaches significantly outperform CPLEX\textsubscript{MIP} and Gurobi. For SAT\textsubscript{Par}, the corresponding overall $p-value$ is approximately 0 for both comparisons. The same conclusion holds for SAT\textsubscript{Inc}, with $p-value$ of approximately 0. 
For the comparison with CPLEX\textsubscript{CP}, although the Wilcoxon tests for individual coefficient values do not always reach the commonly used significance level of 0.05, the results consistently favor the SAT-based approaches. Therefore, when the results for all coefficients are combined, the overall $p-value$ is $0.0051$ for SAT\textsubscript{Par} and approximately 0 for SAT\textsubscript{Inc}, indicating statistically significant improvements over CPLEX\textsubscript{CP}. This confirms that the superior average ranks and solution quality of the SAT-based approaches are not merely due to random variation but also supported by the statistical analysis.

The comparison between SAT\textsubscript{Par} and SAT\textsubscript{Inc} results in an overall $p-value$ of 0.0633. Therefore, although SAT\textsubscript{Inc} achieves slightly better overall results in terms of the number of optimal, best-known solutions, and average rank, the difference between the two SAT strategies cannot be considered statistically significant at the commonly used significance level of 0.05. This suggests that exploiting both parallel CPU resources and incremental SAT solving provides comparable performance for MSABL, with the incremental strategy offering a slight advantage.

For MSCABL, SAT\textsubscript{Par} significantly outperforms all three optimization-based approaches in the overall comparison. The corresponding $p$-values are $0.0121$ for CPLEX\textsubscript{CP} and approximately 0 for both CPLEX\textsubscript{MIP} and Gurobi. These results support the previous analysis, which shows that SAT\textsubscript{Par} achieves both a higher number of optimal and best-known solutions and a better overall average rank than the competing approaches.

Overall, the Wilcoxon signed-rank tests provide statistical support for the observations from the experimental results. Both SAT-based strategies provide statistically significant improvements over CPLEX\textsubscript{CP}, CPLEX\textsubscript{MIP}, and Gurobi on MSABL, while SAT\textsubscript{Par} also significantly outperforms all the other approaches on MSCABL. The difference between the parallel and incremental SAT strategies on MSABL, however, is not statistically significant at the commonly used significance level.

\subsection{Experimental Results on the No-hole Model}

Table~\ref{tab:experimental-results-summary-no-hole} reports the experimental results on the no-hole model for both MSABL and MSCABL. Overall, the SAT-based approaches perform competitively with the optimization-based approaches, particularly in terms of the number of optimal and best-known solutions and the average rank. The results are evaluated for three coefficient values, namely 0.5, 0.75, and 1, with the relative performance of the SAT variants varying slightly between MSABL and MSCABL.

\begin{table*}[ht!]
\renewcommand{\arraystretch}{1.25}
\setlength{\tabcolsep}{3pt}
\centering
\caption{Experimental results on the no-hole model.}
\label{tab:experimental-results-summary-no-hole}
\resizebox{\textwidth}{!}{%
\begin{tabular}{clcccccp{0.25cm}cccc}
\hline
\multirow{2}{*}{\textbf{Coef.}} &  & \multicolumn{5}{c}{\textbf{MSABL}} &  & \multicolumn{4}{c}{\textbf{MSCABL}} \\ \cline{3-7} \cline{9-12} 
 &  & \textbf{CPLEX\textsubscript{CP}} & \textbf{CPLEX\textsubscript{MIP}} & \textbf{Gurobi} & \textbf{SAT\textsubscript{Par}} & \textbf{SAT\textsubscript{Inc}} &  & \textbf{CPLEX\textsubscript{CP}} & \textbf{CPLEX\textsubscript{MIP}} & \textbf{Gurobi} & \textbf{SAT\textsubscript{Par}} \\ \hline
\multirow{4}{*}{0.5} & $\#Solved$ & 24 & 12 & 13 & 24 & 24 &  & 24 & 12 & 12 & 24 \\
 & $\#Optimal$ & 14 & 7 & 8 & \textbf{19} & 18 &  & 14 & 9 & 10 & \textbf{18} \\
 & $\#Best$ & 16 & 7 & 8 & \textbf{20} & 19 &  & 15 & 10 & 11 & \textbf{24} \\
 & $\#Avg.Rank$ & 2.56 & 4.1 & 3.92 & 2.25 & \textbf{2.17} &  & 2.13 & 3.1 & 3.02 & \textbf{1.75} \\ \hline
\multirow{4}{*}{0.75} & $\#Solved$ & \textbf{22} & 11 & 11 & 21 & 19 &  & 22 & 11 & 12 & \textbf{23} \\
 & $\#Optimal$ & 12 & 6 & 6 & \textbf{13} & \textbf{13} &  & 13 & 7 & 6 & \textbf{15} \\
 & $\#Best$ & 15 & 6 & 6 & \textbf{16} & \textbf{16} &  & 15 & 7 & 7 & \textbf{23} \\
 & $\#Avg.Rank$ & 2.38 & 3.92 & 3.92 & \textbf{2.33} & 2.46 &  & 2.06 & 3.17 & 3.1 & \textbf{1.67} \\ \hline
\multirow{4}{*}{1} & $\#Solved$ & \textbf{13} & 4 & 5 & 9 & 9 & \multicolumn{1}{c}{} & \textbf{20} & 6 & 9 & 10 \\
 & $\#Optimal$ & \textbf{7} & 3 & 2 & 4 & 4 & \multicolumn{1}{c}{} & \textbf{11} & 4 & 7 & 9 \\
 & $\#Best$ & \textbf{11} & 3 & 3 & 5 & 5 & \multicolumn{1}{c}{} & \textbf{18} & 4 & 7 & 10 \\
 & $\#Avg.Rank$ & \textbf{2.42} & 3.33 & 3.33 & 2.96 & 2.96 & \multicolumn{1}{c}{} & \textbf{1.75} & 3.06 & 2.71 & 2.48 \\ \hline
\multirow{4}{*}{Overall} & $\#Solved$ & \textbf{59} & 27 & 29 & 54 & 52 & \multicolumn{1}{c}{} & \textbf{66} & 29 & 33 & 57 \\
 & $\#Optimal$ & 33 & 16 & 16 & \textbf{36} & 35 & \multicolumn{1}{c}{} & 38 & 20 & 23 & \textbf{42} \\
 & $\#Best$ & \textbf{42} & 16 & 17 & 41 & 40 & \multicolumn{1}{c}{} & 48 & 21 & 25 & \textbf{57} \\
 & $\#Avg.Rank$ & \textbf{2.45} & 3.78 & 3.72 & 2.51 & 2.53 & \multicolumn{1}{c}{} & 1.98 & 3.11 & 2.94 & \textbf{1.97} \\ \hline
\multicolumn{12}{l}{The best result for each criterion is highlighted in bold.}
\end{tabular}%
}
\end{table*}

For MSABL, SAT\textsubscript{Par} achieves the highest number of optimal solutions for coefficients 0.5 and 0.75, obtaining 19 and 13 optimal solutions, respectively. SAT\textsubscript{Inc} performs similarly, obtaining 18 and 13 optimal solutions for the two coefficients. For coefficient 1, however, CPLEX\textsubscript{CP} performs better, obtaining 7 optimal solutions compared with 4 obtained by each SAT-based approach. Regarding the best-known solutions, SAT\textsubscript{Par} obtains 20, 16, and 5 best-known solutions for coefficients 0.5, 0.75, and 1, respectively, while SAT\textsubscript{Inc} obtains 19, 16, and 5. Overall, SAT\textsubscript{Par} obtains 36 optimal and 41 best-known solutions, whereas SAT\textsubscript{Inc} obtains 35 optimal and 40 best-known solutions. CPLEX\textsubscript{CP} obtains 33 optimal and 42 best-known solutions. These results indicate that the SAT-based approaches are competitive in terms of solution quality, although CPLEX\textsubscript{CP} remains slightly better in the overall number of best-known solutions.

The average-rank results provide a similar evaluation. CPLEX\textsubscript{CP} achieves the best overall average rank of 2.45, followed by SAT\textsubscript{Par} with 2.51 and SAT\textsubscript{Inc} with 2.53. Gurobi and CPLEX\textsubscript{MIP} obtain average ranks of 3.72 and 3.78, respectively. For coefficient 0.5, SAT\textsubscript{Inc} achieves the best average rank of 2.17, while SAT\textsubscript{Par} achieves the best average rank for coefficient 0.75 with a value of 2.33. For coefficient 1, CPLEX\textsubscript{CP} obtains the best average rank of 2.42. In terms of the number of solved instances, CPLEX\textsubscript{CP} solves 59 instances overall, compared with 54 and 52 solved by SAT\textsubscript{Par} and SAT\textsubscript{Inc}, respectively. Thus, although the SAT-based approaches obtain more optimal solutions than CPLEX\textsubscript{CP}, they do not achieve the best overall average rank or the highest number of solved and best-known instances on the no-hole model.

For MSCABL, SAT\textsubscript{Par} provides stronger performance in terms of solution quality. It solves 24, 23, and 10 instances for coefficients 0.5, 0.75, and 1, respectively, resulting in 57 solved instances overall. CPLEX\textsubscript{CP} solves 24, 22, and 20 instances for the three coefficients, resulting in 66 solved instances overall. SAT\textsubscript{Par} achieves the highest number of optimal solutions for coefficients 0.5 and 0.75, with 18 and 15 optimal solutions, respectively, while CPLEX\textsubscript{CP} obtains the highest number of optimal solutions for coefficient 1, with 11 compared to 9 obtained by SAT\textsubscript{Par}. Overall, SAT\textsubscript{Par} obtains 42 optimal and 57 best-known solutions, compared with 38 optimal and 48 best-known solutions obtained by CPLEX\textsubscript{CP}. SAT\textsubscript{Par} also obtains the highest number of best-known solutions for all three coefficients, with 24, 23, and 10, respectively.

The average-rank results further demonstrate the strength of SAT\textsubscript{Par} on MSCABL. It achieves the best average rank for coefficients 0.5 and 0.75, with values of 1.75 and 1.67, respectively. For coefficient 1, however, CPLEX\textsubscript{CP} achieves the best average rank of 1.75, compared with 2.48 obtained by SAT\textsubscript{Par}. Overall, SAT\textsubscript{Par} achieves an average rank of 1.97, slightly better than CPLEX\textsubscript{CP} with 1.98. CPLEX\textsubscript{MIP} and Gurobi obtain overall average ranks of 3.11 and 2.94, respectively. These results show that SAT\textsubscript{Par} provides particularly strong solution quality for MSCABL under the no-hole model, despite CPLEX\textsubscript{CP} solving more instances overall.

Similar to the evaluation on the base model, we apply the Wilcoxon signed-rank test to the paired experimental results on the no-hole model. The results are reported in Table~\ref{tab:wilcoxon-no-hole}, with the $p-value$ rounded to four decimal places. For MSABL, SAT\textsubscript{Par} and SAT\textsubscript{Inc} show statistically significant improvements over CPLEX\textsubscript{MIP} and Gurobi for all three coefficient values, with $p-value$ of approximately 0 in overall. However, neither SAT-based approaches show statistically significant difference from CPLEX\textsubscript{CP}, with the overall $p-value$ of 0.7061 for SAT\textsubscript{Par} and 0.2767 for SAT\textsubscript{Inc}. Thus, although the SAT-based approaches achieve more optimal solutions than CPLEX\textsubscript{CP}, the Wilcoxon test does not provide sufficient evidence to conclude that their overall performance differs significantly from CPLEX\textsubscript{CP} on the no-hole model. In addition, the comparison between SAT\textsubscript{Par} and SAT\textsubscript{Inc} also does not reach statistical significance, with an overall $p-value$ of 0.8334.

\begin{table*}[ht!]
\renewcommand{\arraystretch}{1.25}
\setlength{\tabcolsep}{1.5pt}
\centering
\caption{Wilcoxon signed-rank test results for the no-hole model.}
\label{tab:wilcoxon-no-hole}
\resizebox{\textwidth}{!}{
\begin{tabular}{clccccp{0.25cm}ccccp{0.25cm}ccc}
\hline
\multirow{3}{*}{Coef.} & \multirow{3}{*}{} & \multicolumn{9}{c}{\textbf{MSABL}} &  & \multicolumn{3}{c}{\textbf{MSCABL}} \\ \cline{3-11} \cline{13-15} 
 &  & \multicolumn{4}{c}{\textbf{SAT\textsubscript{Par}}} &  & \multicolumn{4}{c}{\textbf{SAT\textsubscript{Inc}}} &  & \multicolumn{3}{c}{\textbf{SAT\textsubscript{Par}}} \\ \cline{3-6} \cline{8-11} \cline{13-15} 
 &  & \textbf{CPLEX\textsubscript{CP}} & \textbf{CPLEX\textsubscript{MIP}} & \textbf{Gurobi} & \textbf{SAT\textsubscript{Inc}} &  & \textbf{CPLEX\textsubscript{CP}} & \textbf{CPLEX\textsubscript{MIP}} & \textbf{Gurobi} & \textbf{SAT\textsubscript{Par}} &  & \textbf{CPLEX\textsubscript{CP}} & \textbf{CPLEX\textsubscript{MIP}} & \textbf{Gurobi} \\ \hline
\multirow{3}{*}{0.5} & $R^+$ & 45 & 153 & 136 & 8 &  & 51 & 153 & 136 & 13 &  & 45 & 105 & 91 \\
 & $R^-$ & 21 & 0 & 0 & 13 &  & 15 & 0 & 0 & 8 &  & 0 & 0 & 0 \\
 & $p-value$ & 0.2826 & 0.0003 & 0.0004 & 0.5961 &  & 0.1053 & 0.0003 & 0.0004 & 0.5961 &  & 0.0070 & 0.0010 & 0.0015 \\ \hline
\multirow{3}{*}{0.75} & $R^+$ & 33 & 120 & 120 & 18 &  & 21 & 91 & 91 & 10 &  & 36 & 136 & 136 \\
 & $R^-$ & 33 & 0 & 0 & 10 &  & 45 & 0 & 0 & 18 &  & 0 & 0 & 0 \\
 & $p-value$ & 1.0000 & 0.0006 & 0.0006 & 0.4982 &  & 0.2848 & 0.0014 & 0.0014 & 0.4982 &  & 0.0112 & 0.0004 & 0.0004 \\ \hline
\multirow{3}{*}{1} & $R^+$ & 0 & 31 & 10 & 0 &  & 0 & 31 & 10 & 0 &  & 13.5 & 21 & 6 \\
 & $R^-$ & 21 & 5 & 0 & 0 &  & 21 & 5 & 0 & 0 &  & 64.5 & 0 & 0 \\
 & $p-value$ & 0.0273 & 0.0650 & 0.0656 & - &  & 0.0273 & 0.0650 & 0.0656 & - &  & 0.0443 & 0.0273 & 0.1088 \\ \hline
\multirow{3}{*}{Overall} & $R^+$ & 186.5 & 808 & 630 & 48.5 &  & 155.5 & 729 & 561 & 42.5 &  & 289.5 & 666 & 528 \\
 & $R^-$ & 219.5 & 12 & 0 & 42.5 &  & 250.5 & 12 & 0 & 48.5 &  & 145.5 & 0 & 0 \\
 & $p-value$ & 0.7061 & 0.0000 & 0.0000 & 0.8334 &  & 0.2767 & 0.0000 & 0.0000 & 0.8334 &  & 0.1171 & 0.0000 & 0.0000 \\ \hline
\end{tabular}
}
\end{table*}

For MSCABL, SAT\textsubscript{Par} significantly outperforms CPLEX\textsubscript{MIP} and Gurobi in the overall comparison, with an overall $p-value$ of approximately 0 for both comparisons. However, the difference between SAT\textsubscript{Par} and CPLEX\textsubscript{CP} is not statistically significant with overall $p-value$ = 0.1171. Although the individual comparisons with CPLEX\textsubscript{CP} yield $p-value$ below 0.05 for all three coefficients, the relative performance varies across coefficient values. For coefficients 0.5 and 0.75, SAT\textsubscript{Par} performs better, whereas for coefficient 1, CPLEX\textsubscript{CP} performs better. This variation weakens the overall difference, resulting in a non-significant overall $p-value$. Therefore, the statistical analysis supports a significant advantage of SAT\textsubscript{Par} over CPLEX\textsubscript{MIP} and Gurobi, but not over CPLEX\textsubscript{CP}.

Overall, the results demonstrate that the SAT-based approaches remain highly competitive on the no-hole model. For MSABL, SAT\textsubscript{Par} and SAT\textsubscript{Inc} obtain more optimal solutions than CPLEX\textsubscript{CP}, while CPLEX\textsubscript{CP} achieves the best overall average rank and solves the largest number of instances. The Wilcoxon tests further show that the SAT-based approaches are comparable to CPLEX\textsubscript{CP}, while significantly outperforming CPLEX\textsubscript{MIP} and Gurobi. For MSCABL, SAT\textsubscript{Par} achieves substantially more optimal and best-known solutions than the optimization-based approaches and obtains the best overall average rank, although CPLEX\textsubscript{CP} solves more instances. The statistical analysis confirms significant differences between SAT\textsubscript{Par} and both CPLEX\textsubscript{MIP} and Gurobi, but not between SAT\textsubscript{Par} and CPLEX\textsubscript{CP}. These results indicate that the SAT-based approaches are particularly effective in terms of solution quality under the no-hole model, while CPLEX\textsubscript{CP} remains competitive in terms of robustness and the number of solved instances.

\section{Conclusion}
\label{sec:conclusion}

In this paper, we introduce a unified SAT-based framework for the Minimum Span Antibandwidth Labeling (MSABL) and Minimum Span Cyclic Antibandwidth Labeling (MSCABL) problems. These problems provide a complementary perspective to the conventional Antibandwidth and Cyclic Antibandwidth problems by minimizing the label span required to satisfy a prescribed (cyclic) distance threshold. We formulated both problems as monotone decision problems and exploited this property to design an effective search strategy for determining the minimum feasible span. For both problems, we investigated parallel SAT solving, while incremental SAT solving was additionally considered for MSABL.

Extensive computational experiments on the Harwell-Boeing benchmark demonstrate the effectiveness of the proposed SAT-based approaches. On the base model, the SAT-based methods generally achieve stronger solution quality and competitive average ranks compared with the considered optimization-based approaches for both MSABL and MSCABL. In particular, SAT\textsubscript{Inc} achieves the best overall performance for MSABL, while SAT\textsubscript{Par} performs best overall for MSCABL. The Wilcoxon signed-rank tests provide statistical support for these observations, showing significant overall differences in favor of the SAT-based approaches in most comparisons. The experiments on the no-hole model lead to similar conclusions, that is, the SAT-based approaches remain competitive in terms of solution quality, while CPLEX\textsubscript{CP} remains competitive in terms of the number of solved instances in some cases.

Despite these promising results, several limitations remain. First, the incremental SAT strategy is currently applicable only to MSABL, whereas an effective incremental formulation for MSCABL remains to be developed. Second, the SAT encoding obtained from the underlying optimization models could be further improved to reduce its size and strengthen propagation. These limitations suggest several directions for future work, including the development of incremental techniques for MSCABL, more compact and propagation-efficient SAT encodings, stronger symmetry-breaking constraints, and adaptive strategies for span selection and computational resource allocation.

Overall, this work demonstrates the potential of SAT solving as an effective exact approach for minimum span graph labeling problems. By introducing SAT-based solution methods for MSABL and MSCABL, we provide a computational framework for studying the minimum span perspective of Antibandwidth and Cyclic Antibandwidth problems and establish a basis for further investigation of SAT-based methods for these and related graph labeling problems.







\begin{appendices}

\section{Monotonicity of the decision problem of the Minimum Span Labeling problem}
\label{appendix-monotonicity-property}

The SAT-based approach presented in this paper solves the Minimum Span Labeling problem as a sequence of decision problems with a fixed largest label $\lambda$. Proposition~\ref{prop:monotonocity-property} establishes that the feasibility of the decision problem is monotone, i.e., if a labeling exists for $\lambda$, then it also exists for every $\lambda'
\ge \lambda$. Equivalently, by contraposition, if the problem is infeasible for $\lambda$, then it is also infeasible for every $\lambda'' \le \lambda$.

\begin{proposition}
Given an undirected graph $G=(V,E)$, a fixed (cyclic) distance value $k$, and a candidate largest label $\lambda$, suppose that there exists a feasible labeling
\[
f:V\rightarrow\{1,\ldots,\lambda\}.
\]
Then, for every $\lambda' \ge \lambda$, there also exists a feasible labeling
\[
f':V\rightarrow\{1,\ldots,\lambda'\}.
\]
\label{prop:monotonocity-property}
\end{proposition}

\begin{proof}
Suppose that
\[
f:V\rightarrow\{1,\ldots,\lambda\}
\]
is a feasible labeling. Since label $\lambda$ is required to be used, there exists at least one vertex $u\in V$ such that $f(u)=\lambda$. We construct a new labeling $f'$ by setting

\[
\left\{
\begin{aligned}
&f'(u)=\lambda', \qquad \lambda' \geq \lambda \\
&f'(v)=f(v), \qquad \forall v\in V\setminus\{u\}.
\end{aligned}
\right.
\]

Clearly,
\[
f':V\rightarrow\{1,\ldots,\lambda'\}
\]
is a valid labeling. Moreover, label $1$ remains assigned to at least one vertex, while label $\lambda'$ is assigned to $u$. Hence, \eqref{eq:labels}, \eqref{eq:min-label}, and \eqref{eq:max-label} are satisfied.

For every edge $\{x,y\}\in E$ with $x,y\in V\setminus\{u\}$, the assigned labels remain unchanged. Therefore, 
\[
|f'(x) - f'(y)| = |f(x) - f(y)|.
\]
Therefore
\[
\lambda' - |f'(x) - f'(y)| \ge \lambda - |f(x) - f(y)|,
\]
and
\[
\min\big( |f'(x) - f'(y)|,~\lambda' -  |f'(x) - f'(y)| \big) \ge \min\big( |f(x) - f(y)|,~\lambda -  |f(x) - f(y)| \big).
\]
Hence, \eqref{eq:distance-k} and \eqref{eq:cyclic-distance-k} remain satisfied on every edge not incident to $u$.

Now consider an edge $\{u,v\}\in E$. We have
\[
|f'(u)-f'(v)|=\lambda'-f'(v)\ge\lambda-f(v)=|f(u)-f(v)|,
\]
and
\[
\lambda' - |f'(u) - f'(v)| = \lambda' - |\lambda' - f'(v)| = f'(v) = f(v) = \lambda - |\lambda - f(v)| = \lambda - |f(u) - f(v)|.
\]
Therefore,
\[
\min\big( |f'(u) - f'(v)|, \lambda' -  |f'(u) - f'(v)| \big) \ge \min\big( |f(u) - f(v)|, \lambda -  |f(u) - f(v)| \big).
\]
Hence, \eqref{eq:distance-k} and \eqref{eq:cyclic-distance-k} are also preserved on every edge incident to $u$. 

Consequently, $f'$ satisfies all constraints of the decision problem, proving the proposition.
\end{proof}

\section{Applicability of Incremental SAT Solving to MSABL and MSCABL}
\label{sec:applicability-of-inc-sat-to-msl}

The incremental SAT approach exploits the fact that, for a fixed $k$, decreasing the candidate largest label $\lambda$ only restricts the labeling domain. In MSABL, given a labeling $f$, the distance between two labels of vertex $u$ and vertex $v$ is given by 
\[|f(u) - f(v)|,\]
which is independent of $\lambda$. Therefore, when $\lambda$ is decreased to $\lambda - 1$, the distance constraints remain unchanged, and the new instance can be obtained from the previous one by adding clauses that forbid the use of label $\lambda$. In particular, for every vertex $v$, we only need to add the clause
\[ \neg x_{v,\lambda}, \]
with  $x_{v, l}$ be $true$ if and only if vertex $v$ is assigned label $l$, and $false$ otherwise.
Hence, the SAT formula for a smaller candidate span can be obtained incrementally from the formula constructed for the larger span.

However, this property does not hold for MSCABL. It is because the cyclic distance depends explicitly on the size of the label domain. Given a labeling $f$, for two labels $f(u)$ and $f(v)$, the cyclic distance under a largest label $\lambda$ is defined as
\[ \min\{|a-b|,\lambda-|a-b|\}. \]
Consequently, changing $\lambda$ changes not only the available label set but also the distance between the labels that remain in the domain.

As a result, in the MSCABL, the SAT encoding constructed for $\lambda$ cannot be transformed into the encoding for $\lambda-1$ merely by adding clauses that exclude label $\lambda$. The clauses encoding the cyclic distance constraints may themselves need to be modified. Therefore, although the incremental SAT approach described above can be directly applied to MSABL, it is currently not applicable to MSCABL.

\section{Label Domain of the Cyclic Distance Constraint}
\label{sec:append-label-domain-cab}

Different from the distance constraint, the cyclic distance constraint depends not only on the labeling but also on the label domain. More precisely, for a label domain
\[
\{L,L+1,\ldots,\lambda\} \in \mathbb{Z},
\]
the perimeter of the underlying cycle is determined by the cardinality of the domain, namely \(\lambda-L+1\). Consequently, translating all labels by a constant changes the label domain and alters the cyclic distances between labels.

Figure~\ref{fig:dependence-label-domain-cyclic} illustrates this phenomenon by considering an edge \(\{u,v\}\), where \(u\) is assigned the smallest label and \(v\) the largest label of the domain. The four subfigures represent four translated versions of the same relative labeling. Although the relative positions of the labels are identical in all four cases, the computed cyclic distance is not translation invariant.

\begin{figure}
\centering

\begin{tikzpicture}[
    vertex/.style={circle,draw,fill=white,minimum size=7mm,inner sep=0pt,font=\small},
    every node/.style={font=\small},
    >=Latex
]

\def\r{1.25}
\def\ra{2}


\begin{scope}[shift={(0,0)}]

\node at (0,2.75) {\textbf{Label domain $\{-1, \ldots, 3\}$}};

\coordinate (C) at (0,0);

\draw ($(C)+(0:\r)$)
arc[start angle=0,end angle=360,radius=\r];

\foreach \a/\l/\name in {90/1/P1,18/2/P2,-54/3/P3,-126/4/P4,162/5/P5}
{
    \coordinate (\name) at ($(C)+(\a:\r)$);
    \fill (\name) circle (1.2pt);
}

\foreach \a/\l in {-90/-1,-162/0,126/1,54/2,-18/3}
    \node at ($(C)+(\a:-1)$) {\l};

\draw[dashed,red]
($(C)+(90:\ra)$)
arc[start angle=90,end angle=162,radius=\ra];
\node[left,red] at ($(C)+(126:{\ra+0.25})$) {$-1$};


\node[vertex] (U1) at ($(C)+(90:2)$) {$u$};
\node[vertex,draw=red] (V11) at ($(C)+(162:2)$) {$v$};

\draw[dotted] (U1)--(P1);
\draw[dotted] (V11)--(P5);

\end{scope}


\begin{scope}[shift={(8,0)}]

\node at (0,2.75) {\textbf{Label domain $\{0, \ldots, 4\}$}};

\coordinate (C) at (0,0);

\draw ($(C)+(0:\r)$)
arc[start angle=0,end angle=360,radius=\r];

\foreach \a/\l/\name in {90/1/P1,18/2/P2,-54/3/P3,-126/4/P4,162/5/P5}
{
    \coordinate (\name) at ($(C)+(\a:\r)$);
    \fill (\name) circle (1.2pt);
}

\foreach \a/\l in {-90/0,-162/1,126/2,54/3,-18/4}
    \node at ($(C)+(\a:-1)$) {\l};

\draw[dashed,red]
($(C)+(90:\ra)$)
arc[start angle=90,end angle=162,radius=\ra];
\node[left,red] at ($(C)+(126:{\ra+0.25})$) {$0$};


\node[vertex] (U2) at ($(C)+(90:2)$) {$u$};
\node[vertex,draw=red] (V21) at ($(C)+(162:2)$) {$v$};

\draw[dotted] (U2)--(P1);
\draw[dotted] (V21)--(P5);

\end{scope}


\begin{scope}[shift={(0,-5)}]

\node at (0,2.75) {\textbf{Label domain $\{1, \ldots, 5\}$}};

\coordinate (C) at (0,0);

\draw ($(C)+(0:\r)$)
arc[start angle=0,end angle=360,radius=\r];

\foreach \a/\l/\name in {90/1/P1,18/2/P2,-54/3/P3,-126/4/P4,162/5/P5}
{
    \coordinate (\name) at ($(C)+(\a:\r)$);
    \fill (\name) circle (1.2pt);
}

\foreach \a/\l in {-90/1,-162/2,126/3,54/4,-18/5}
    \node at ($(C)+(\a:-1)$) {\l};

\draw[dashed,green!60!black]
($(C)+(90:\ra)$)
arc[start angle=90,end angle=162,radius=\ra];
\node[left,green!60!black] at ($(C)+(126:{\ra+0.25})$) {$1$};


\node[vertex] (U3) at ($(C)+(90:2)$) {$u$};
\node[vertex,draw=green!60!black] (V31) at ($(C)+(162:2)$) {$v$};

\draw[dotted] (U3)--(P1);
\draw[dotted] (V31)--(P5);

\end{scope}


\begin{scope}[shift={(8,-5)}]

\node at (0,2.75) {\textbf{Label domain $\{2, \ldots, 6\}$}};

\coordinate (C) at (0,0);

\draw ($(C)+(0:\r)$)
arc[start angle=0,end angle=360,radius=\r];

\foreach \a/\l/\name in {90/1/P1,18/2/P2,-54/3/P3,-126/4/P4,162/5/P5}
{
    \coordinate (\name) at ($(C)+(\a:\r)$);
    \fill (\name) circle (1.2pt);
}

\foreach \a/\l in {-90/2,-162/3,126/4,54/5,-18/6}
    \node at ($(C)+(\a:-1)$) {\l};

\draw[dashed,red]
($(C)+(90:\ra)$)
arc[start angle=90,end angle=162,radius=\ra];
\node[left,red] at ($(C)+(126:{\ra+0.25})$) {$2$};


\node[vertex] (U4) at ($(C)+(90:2)$) {$u$};
\node[vertex,draw=red] (V41) at ($(C)+(162:2)$) {$v$};

\draw[dotted] (U4)--(P1);
\draw[dotted] (V41)--(P5);

\end{scope}


\draw[dashed,green!60!black] (-1,-7.25) -- +(1,0);
\node[right] at (0.25,-7.25) {Correct cyclic distance};

\draw[dashed,red] (5,-7.25) -- +(1,0);
\node[right] at (6.25,-7.25) {Incorrect cyclic distance};

\end{tikzpicture}

\caption{Dependence of the cyclic distance on the label domain.}
\label{fig:dependence-label-domain-cyclic}
\end{figure}
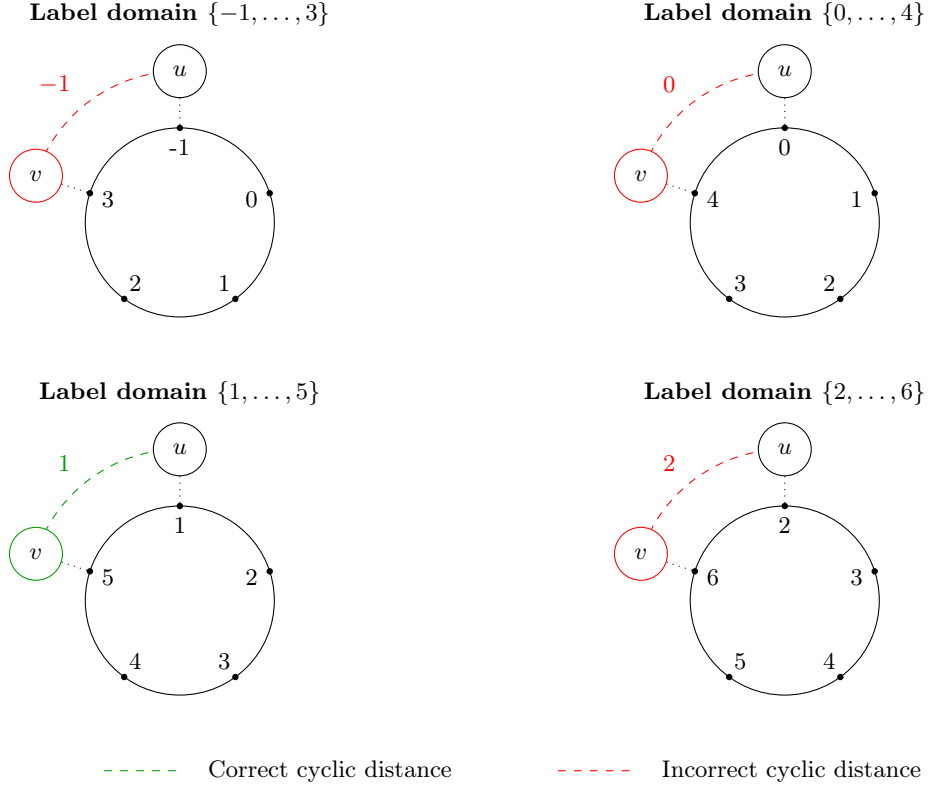

The corresponding numerical values are reported in Table~\ref{table:cab-u-v-in-dif-domains}. Although the absolute difference
\[
|f(u)-f(v)|=4
\]
remains unchanged across all four label domains, the value of 
\[
\lambda-|f(u)-f(v)|
\]
varies.
As a result, the computed cyclic distance is \(-1\), \(0\), \(1\), and \(2\) for the label domains \(\{-1,\ldots,3\}\), \(\{0,\ldots,4\}\), \(\{1,\ldots,5\}\), and \(\{2,\ldots,6\}\), respectively. Among these four cases, only the label domain \(\{1,\ldots,5\}\) yields the intended cyclic distance of \(1\), whereas the remaining three domains produce incorrect values.

\begin{table*}[ht!]
\caption{Cyclic antibandwidth value of edge $\{u, v\}$ in different label domains.}
\label{table:cab-u-v-in-dif-domains}
\renewcommand{\arraystretch}{1.5}
\resizebox{\textwidth}{!}{
\begin{tabular}{cccccccc}
\toprule
Label domain & $\lambda$ & $f(u)$ & $f(v)$ & $|f(u) - f(v)|$ & $\lambda - |f(u) - f(v)|$ & $D_c(u,v,f)$ & Is correct? \\ 
\midrule
$\{-1, \ldots, 3\}$ & 3 & -1 & 3 & 4 & -1 & -1 & \textcolor{red}{\ding{55}} \\
$\{0, \ldots, 4\}$ & 4 & 0 & 4 & 4 & 0 & 0 & \textcolor{red}{\ding{55}} \\ 
$\{1, \ldots, 5\}$ & 5 & 1 & 5 & 4 & 1 & 1 & \textcolor{green}{\ding{51}}  \\ 
$\{2, \ldots, 6\}$ & 6 & 2 & 6 & 4 & 2 & 2 & \textcolor{red}{\ding{55}} \\ 
\bottomrule
\end{tabular}
}
\end{table*}

Since the cyclic distance constraint is not invariant under translations of the labeling, one cannot assume, without loss of generality, that the minimum label is fixed to an arbitrary value such as \(0\). To obtain a unique and consistent definition of the cyclic distance, the label domain must be fixed to
\[
\{1,\ldots,\lambda\},
\]
where \(\lambda\) denotes the largest assigned label.

\section{Detailed Experimental Results}
\label{sec:detailed-experiment-results}

This section presents the detailed experimental results. For the base model, the results for coefficients 0.5, 0.75, 1, 1.25, and 1.5 are reported in Table~\ref{tab:experimental-results-0.5}, \ref{tab:experimental-results-0.75}, \ref{tab:experimental-results-1}, \ref{tab:experimental-results-1.25}, and \ref{tab:experimental-results-1.5}, respectively. For the no-hole model, the results for coefficients 0.5, 0.75, and 1 are reported in Tables~\ref{tab:experimental-results-0.5_no_hole}, \ref{tab:experimental-results-0.75-no-hole}, and \ref{tab:experimental-results-1-no-hole}, respectively. In these tables, the best result for each instance and comparison criterion is highlighted in bold. Additionally, an asterisk (*) indicates an optimal result, while a dash (-) indicates a timeout instance for which no feasible span is found.

\newcommand{\opt}[1]{#1\rlap{$^{\text{\fontsize{9}{10}\selectfont *}}$}}

\begin{table*}[ht!]
\renewcommand{\arraystretch}{1.25}
\setlength{\tabcolsep}{3pt}
\centering
\caption{Experimental results with coefficient 0.5.}
\label{tab:experimental-results-0.5}
\resizebox{\textwidth}{!}{
\begin{tabular}{lcp{0.25cm}cccccp{0.25cm}cccc}
\hline
\multicolumn{1}{c}{\multirow{2}{*}{\textbf{Graph}}} & \multirow{2}{*}{\textbf{|V|}} &  & \multicolumn{5}{c}{\textbf{MSABL}} &  & \multicolumn{4}{c}{\textbf{MSCABL}} \\ \cline{4-8} \cline{10-13} 
\multicolumn{1}{c}{} &  &  & \textbf{CPLEX\textsubscript{CP}} & \textbf{CPLEX\textsubscript{MIP}} & \textbf{Gurobi} & \textbf{SAT\textsubscript{Par}} & \textbf{SAT\textsubscript{Inc}} &  & \textbf{CPLEX\textsubscript{CP}} & \textbf{CPLEX\textsubscript{MIP}} & \textbf{Gurobi} & \textbf{SAT\textsubscript{Par}} \\ \hline
A-pores\_1 & 30 &  & \opt{9} & \opt{9} & \opt{9} & \opt{9} & \opt{9} &  & \opt{12} & \opt{12} & \opt{12} & \opt{12} \\
B-ibm32 & 32 &  & \opt{12} & \opt{12} & \opt{12} & \opt{12} & \opt{12} &  & \opt{15} & \opt{15} & \opt{15} & \opt{15} \\
C-bcspwr01 & 39 &  & \opt{16} & \opt{16} & \opt{16} & \opt{16} & \opt{16} &  & \opt{17} & \opt{17} & \opt{17} & \opt{17} \\
D-bcsstk01 & 48 &  & \opt{20} & 20 & \opt{20} & \opt{20} & \opt{20} &  & \opt{23} & 23 & \opt{23} & \opt{23} \\
E-bcspwr02 & 49 &  & \opt{20} & \opt{20} & \opt{20} & \opt{20} & \opt{20} &  & \opt{23} & \opt{23} & \opt{23} & \opt{23} \\
F-curtis54 & 54 &  & \opt{24} & 24 & \opt{24} & \opt{24} & \opt{24} &  & \opt{24} & 24 & \opt{24} & \opt{24} \\
G-will57 & 57 &  & \opt{24} & 24 & \opt{24} & \opt{24} & \opt{24} &  & \opt{24} & \opt{24} & \opt{24} & \opt{24} \\
H-impcol\_b & 59 &  & \opt{28} & 28 & 28 & \opt{28} & \opt{28} &  & \opt{23} & 23 & 23 & \opt{23} \\
I-ash85 & 85 &  & \opt{33} & \opt{33} & \opt{33} & \opt{33} & \opt{33} &  & \opt{39} & \opt{39} & \opt{39} & \opt{39} \\
J-nos4 & 100 &  & \opt{34} & \opt{34} & \opt{34} & \opt{34} & \opt{34} &  & \opt{47} & \opt{47} & \opt{47} & \opt{47} \\
K-dwt\_\_234 & 117 &  & \opt{50} & \opt{50} & \opt{50} & \opt{50} & \opt{50} &  & \opt{55} & \opt{55} & \opt{55} & \opt{55} \\
L-bcspwr03 & 118 &  & \opt{57} & \opt{57} & 57 & \opt{57} & \opt{57} &  & \opt{55} & 55 & 55 & \opt{55} \\
M-bcsstk06 & 420 &  & 187 & - & - & 188 & 187 &  & 191 & - & - & 191 \\
N-bcsstk07 & 420 &  & 187 & - & - & 188 & 187 &  & 191 & - & - & 191 \\
O-impcol\_d & 425 &  & \opt{180} & - & - & \opt{180} & \opt{180} &  & \opt{207} & - & - & \opt{207} \\
P-can\_\_445 & 445 &  & 180 & - & - & \textbf{\opt{180}} & \textbf{\opt{180}} &  & 213 & - & - & \textbf{\opt{181}} \\
Q-494\_bus & 494 &  & \opt{226} & \opt{226} & \opt{226} & \opt{226} & \opt{226} &  & \opt{245} & \opt{245} & \opt{245} & \opt{245} \\
R-dwt\_\_503 & 503 &  & 217 & - & - & \textbf{\opt{217}} & \textbf{\opt{217}} &  & 247 & - & - & \textbf{\opt{247}} \\
S-sherman4 & 546 &  & \opt{130} & \opt{130} & \opt{130} & \opt{130} & \opt{130} &  & \opt{258} & - & \opt{258} & \opt{258} \\
T-dwt\_\_592 & 592 &  & \opt{168} & - & - & \opt{168} & \opt{168} &  & \opt{224} & - & - & \opt{224} \\
U-662\_bus & 662 &  & \opt{330} & 330 & 330 & \opt{330} & \opt{330} &  & \opt{327} & 329 & \opt{327} & \opt{327} \\
V-nos6 & 675 &  & \opt{164} & - & \opt{164} & \opt{164} & \opt{164} &  & \opt{328} & - & 329 & \opt{328} \\
W-685\_bus & 685 &  & 340 & 340 & 341 & \textbf{\opt{340}} & \textbf{\opt{340}} &  & 341 & - & 341 & \textbf{\opt{341}} \\
X-can\_\_715 & 715 &  & 290 & - & - & \textbf{\opt{290}} & \textbf{\opt{290}} &  & 299 & - & - & \textbf{\opt{299}} \\ \hline
\multicolumn{2}{l}{$\#Solved$} &  & 24 & 16 & 17 & 24 & 24 &  & 24 & 14 & 17 & 24 \\
\multicolumn{2}{l}{$\#Optimal$} &  & 18 & 10 & 13 & \textbf{22} & \textbf{22} &  & 18 & 9 & 13 & \textbf{22} \\
\multicolumn{2}{l}{$\#Best$} &  & 20 & 10 & 13 & 22 & \textbf{24} &  & 20 & 9 & 13 & \textbf{24} \\
\multicolumn{2}{l}{$\#Avg.Rank$} &  & 2.63 & 3.9 & 3.65 & 2.48 & \textbf{2.35} &  & 2.06 & 3.25 & 2.81 & \textbf{1.88} \\ \hline
\multicolumn{13}{l}{The best result for each instance and comparison criterion is highlighted in bold.} \\
\multicolumn{13}{l}{An asterisk (*) indicates an optimal result.} \\
\multicolumn{13}{l}{A dash (-) indicates a timeout instance with no feasible span found.}
\end{tabular}
}
\end{table*}

\begin{table*}[ht!]
\renewcommand{\arraystretch}{1.25}
\setlength{\tabcolsep}{3pt}
\centering
\caption{Experimental results with coefficient 0.75.}
\label{tab:experimental-results-0.75}
\resizebox{\textwidth}{!}{
\begin{tabular}{lcp{0.25cm}cccccp{0.25cm}cccc}
\hline
\multicolumn{1}{c}{\multirow{2}{*}{\textbf{Graph}}} & \multirow{2}{*}{\textbf{|V|}} &  & \multicolumn{5}{c}{\textbf{MSABL}} &  & \multicolumn{4}{c}{\textbf{MSCABL}} \\ \cline{4-8} \cline{10-13} 
\multicolumn{1}{c}{} &  &  & \textbf{CPLEX\textsubscript{CP}} & \textbf{CPLEX\textsubscript{MIP}} & \textbf{Gurobi} & \textbf{SAT\textsubscript{Par}} & \textbf{SAT\textsubscript{Inc}} &  & \textbf{CPLEX\textsubscript{CP}} & \textbf{CPLEX\textsubscript{MIP}} & \textbf{Gurobi} & \textbf{SAT\textsubscript{Par}} \\ \hline
A-pores\_1 & 30 &  & \opt{12} & \opt{12} & \opt{12} & \opt{12} & \opt{12} &  & \opt{16} & \opt{16} & \opt{16} & \opt{16} \\
B-ibm32 & 32 &  & \opt{18} & \opt{18} & \opt{18} & \opt{18} & \opt{18} &  & \opt{23} & \opt{23} & \opt{23} & \opt{23} \\
C-bcspwr01 & 39 &  & \opt{24} & \opt{24} & \opt{24} & \opt{24} & \opt{24} &  & \opt{26} & \opt{26} & \opt{26} & \opt{26} \\
D-bcsstk01 & 48 &  & \opt{30} & 30 & \opt{30} & \opt{30} & \opt{30} &  & \opt{35} & \opt{35} & \opt{35} & \opt{35} \\
E-bcspwr02 & 49 &  & \opt{30} & \opt{30} & \opt{30} & \opt{30} & \opt{30} &  & \opt{35} & \opt{35} & \opt{35} & \opt{35} \\
F-curtis54 & 54 &  & \opt{36} & \opt{36} & \opt{36} & \opt{36} & \opt{36} &  & \opt{34} & \opt{34} & \opt{34} & \opt{34} \\
G-will57 & 57 &  & \opt{36} & \opt{36} & \opt{36} & \opt{36} & \opt{36} &  & \opt{39} & \opt{39} & \opt{39} & \opt{39} \\
H-impcol\_b & 59 &  & \opt{42} & 42 & 42 & \opt{42} & \opt{42} &  & \opt{39} & 39 & 39 & \opt{39} \\
I-ash85 & 85 &  & \opt{51} & \opt{51} & \opt{51} & \opt{51} & \opt{51} &  & \opt{59} & \opt{59} & \opt{59} & \opt{59} \\
J-nos4 & 100 &  & \opt{52} & \opt{52} & \opt{52} & \opt{52} & \opt{52} &  & \opt{71} & \opt{71} & \opt{71} & \opt{71} \\
K-dwt\_\_234 & 117 &  & \opt{76} & \opt{76} & \opt{76} & \opt{76} & \opt{76} &  & \opt{81} & 81 & \opt{81} & \opt{81} \\
L-bcspwr03 & 118 &  & \opt{87} & \opt{87} & \opt{87} & \opt{87} & \opt{87} &  & \opt{83} & \opt{83} & 83 & \opt{83} \\
M-bcsstk06 & 420 &  & 275 & - & - & 275 & 275 &  & 287 & - & - & 287 \\
N-bcsstk07 & 420 &  & 275 & - & - & 275 & 275 &  & 287 & - & - & 287 \\
O-impcol\_d & 425 &  & \opt{270} & - & - & \opt{270} & \opt{270} &  & \opt{311} & - & - & \opt{311} \\
P-can\_\_445 & 445 &  & 268 & - & - & \textbf{\opt{268}} & \textbf{\opt{268}} &  & 324 & - & - & \textbf{\opt{274}} \\
Q-494\_bus & 494 &  & \opt{340} & 365 & \opt{340} & \opt{340} & \opt{340} &  & \opt{368} & 369 & \opt{368} & \opt{368} \\
R-dwt\_\_503 & 503 &  & 329 & - & - & \textbf{\opt{329}} & \textbf{\opt{329}} &  & 367 & - & - & \textbf{\opt{367}} \\
S-sherman4 & 546 &  & \opt{195} & \opt{195} & \opt{195} & \opt{195} & \opt{195} &  & \opt{386} & - & \opt{386} & \opt{386} \\
T-dwt\_\_592 & 592 &  & \opt{252} & - & - & \opt{252} & \opt{252} &  & \opt{336} & - & - & \opt{336} \\
U-662\_bus & 662 &  & \opt{495} & - & 495 & \opt{495} & \opt{495} &  & \opt{491} & - & - & \opt{491} \\
V-nos6 & 675 &  & \opt{246} & - & - & \opt{246} & \opt{246} &  & \opt{492} & - & \opt{492} & \opt{492} \\
W-685\_bus & 685 &  & 510 & - & 510 & \textbf{\opt{510}} & \textbf{\opt{510}} &  & 509 & - & 509 & \textbf{\opt{509}} \\
X-can\_\_715 & 715 &  & 435 & - & - & \textbf{\opt{435}} & \textbf{\opt{435}} &  & 449 & - & - & \textbf{\opt{449}} \\ \hline
\multicolumn{2}{l}{$\#Solved$} &  & 24 & 14 & 16 & 24 & 24 &  & 24 & 13 & 16 & 24 \\
\multicolumn{2}{l}{$\#Optimal$} &  & 18 & 11 & 13 & \textbf{22} & \textbf{22} &  & 18 & 10 & 13 & \textbf{22} \\
\multicolumn{2}{l}{$\#Best$} &  & 20 & 11 & 13 & \textbf{24} & \textbf{24} &  & 20 & 10 & 13 & \textbf{24} \\
\multicolumn{2}{l}{$\#Avg.Rank$} &  & 2.69 & 3.9 & 3.58 & \textbf{2.42} & \textbf{2.42} &  & 2.08 & 3.17 & 2.85 & \textbf{1.9} \\ \hline
\multicolumn{13}{l}{The best result for each instance and comparison criterion is highlighted in bold.} \\
\multicolumn{13}{l}{An asterisk (*) indicates an optimal result.} \\
\multicolumn{13}{l}{A dash (-) indicates a timeout instance with no feasible span found.}
\end{tabular}
}
\end{table*}

\begin{table*}[ht!]
\renewcommand{\arraystretch}{1.25}
\setlength{\tabcolsep}{3pt}
\centering
\caption{Experimental results with coefficient 1.}
\label{tab:experimental-results-1}
\resizebox{\textwidth}{!}{
\begin{tabular}{lcp{0.25cm}cccccp{0.25cm}cccc}
\hline
\multicolumn{1}{c}{\multirow{2}{*}{\textbf{Graph}}} & \multirow{2}{*}{\textbf{|V|}} &  & \multicolumn{5}{c}{\textbf{MSABL}} &  & \multicolumn{4}{c}{\textbf{MSCABL}} \\ \cline{4-8} \cline{10-13} 
\multicolumn{1}{c}{} &  &  & \textbf{CPLEX\textsubscript{CP}} & \textbf{CPLEX\textsubscript{MIP}} & \textbf{Gurobi} & \textbf{SAT\textsubscript{Par}} & \textbf{SAT\textsubscript{Inc}} &  & \textbf{CPLEX\textsubscript{CP}} & \textbf{CPLEX\textsubscript{MIP}} & \textbf{Gurobi} & \textbf{SAT\textsubscript{Par}} \\ \hline
A-pores\_1 & 30 &  & \opt{18} & \opt{18} & \opt{18} & \opt{18} & \opt{18} &  & \opt{24} & \opt{24} & \opt{24} & \opt{24} \\
B-ibm32 & 32 &  & \opt{27} & \opt{27} & \opt{27} & \opt{27} & \opt{27} &  & \opt{31} & \opt{31} & \opt{31} & \opt{31} \\
C-bcspwr01 & 39 &  & \opt{34} & \opt{34} & \opt{34} & \opt{34} & \opt{34} &  & \opt{38} & \opt{38} & \opt{38} & \opt{38} \\
D-bcsstk01 & 48 &  & \opt{45} & 45 & \opt{45} & \opt{45} & \opt{45} &  & \opt{47} & 47 & \opt{47} & \opt{47} \\
E-bcspwr02 & 49 &  & \opt{42} & \opt{42} & \opt{42} & \opt{42} & \opt{42} &  & \opt{47} & \opt{47} & \opt{47} & \opt{47} \\
F-curtis54 & 54 &  & \opt{52} & \opt{52} & \opt{52} & \opt{52} & \opt{52} &  & \opt{49} & \opt{49} & \opt{49} & \opt{49} \\
G-will57 & 57 &  & \opt{52} & \opt{52} & \opt{52} & \opt{52} & \opt{52} &  & \opt{54} & \opt{54} & 54 & \opt{54} \\
H-impcol\_b & 59 &  & \opt{56} & 56 & \opt{56} & \opt{56} & \opt{56} &  & \opt{55} & 55 & 55 & \opt{55} \\
I-ash85 & 85 &  & \opt{69} & \opt{69} & \opt{69} & \opt{69} & \opt{69} &  & \opt{83} & \opt{83} & \opt{83} & \opt{83} \\
J-nos4 & 100 &  & \opt{70} & \opt{70} & \opt{70} & \opt{70} & \opt{70} &  & \opt{95} & \opt{95} & \opt{95} & \opt{95} \\
K-dwt\_\_234 & 117 &  & \opt{102} & \opt{102} & \opt{102} & \opt{102} & \opt{102} &  & \opt{110} & \opt{110} & \opt{110} & \opt{110} \\
L-bcspwr03 & 118 &  & \opt{117} & \opt{117} & \opt{117} & \opt{117} & \opt{117} &  & \opt{115} & \opt{115} & 115 & \opt{115} \\
M-bcsstk06 & 420 &  & 374 & - & 408 & 374 & 374 &  & 395 & - & - & 395 \\
N-bcsstk07 & 420 &  & 374 & - & 408 & 374 & 374 &  & 395 & - & - & 395 \\
O-impcol\_d & 425 &  & \opt{360} & - & \opt{360} & \opt{360} & \opt{360} &  & \opt{419} & - & - & \opt{419} \\
P-can\_\_445 & 445 &  & 360 & - & 360 & \textbf{\opt{360}} & \textbf{\opt{360}} &  & 433 & - & - & \textbf{367} \\
Q-494\_bus & 494 &  & \opt{454} & 493 & \opt{454} & \opt{454} & \opt{454} &  & \opt{491} & 493 & \opt{491} & \opt{491} \\
R-dwt\_\_503 & 503 &  & 441 & - & 441 & \textbf{\opt{441}} & \textbf{\opt{441}} &  & 495 & - & - & \textbf{\opt{495}} \\
S-sherman4 & 546 &  & \opt{261} & - & \opt{261} & \opt{261} & \opt{261} &  & \opt{516} & 542 & \opt{516} & \opt{516} \\
T-dwt\_\_592 & 592 &  & \opt{339} & - & \opt{339} & \opt{339} & \opt{339} &  & \opt{452} & - & - & \opt{452} \\
U-662\_bus & 662 &  & \opt{660} & - & 660 & \opt{660} & \opt{660} &  & \opt{659} & - & 660 & \opt{659} \\
V-nos6 & 675 &  & \opt{329} & \opt{329} & \opt{329} & \opt{329} & \opt{329} &  & \opt{656} & - & 657 & \opt{656} \\
W-685\_bus & 685 &  & 680 & - & 680 & \textbf{\opt{680}} & \textbf{\opt{680}} &  & 683 & - & 684 & \textbf{\opt{683}} \\
X-can\_\_715 & 715 &  & 580 & - & 580 & \textbf{\opt{580}} & \textbf{\opt{580}} &  & 605 & - & - & \textbf{\opt{605}} \\ \hline
\multicolumn{2}{l}{$\#Solved$} &  & 24 & 14 & 24 & 24 & 24 &  & 24 & 14 & 17 & 24 \\
\multicolumn{2}{l}{$\#Optimal$} &  & 18 & 11 & 17 & \textbf{22} & \textbf{22} &  & 18 & 10 & 11 & \textbf{21} \\
\multicolumn{2}{l}{$\#Best$} &  & 20 & 11 & 17 & \textbf{24} & \textbf{24} &  & 20 & 10 & 11 & \textbf{24} \\
\multicolumn{2}{l}{$\#Avg.Rank$} &  & 2.83 & 4.08 & 3.08 & \textbf{2.5} & \textbf{2.5} &  & 2.02 & 3.17 & 2.96 & \textbf{1.85} \\ \hline
\multicolumn{13}{l}{The best result for each instance and comparison criterion is highlighted in bold.} \\
\multicolumn{13}{l}{An asterisk (*) indicates an optimal result.} \\
\multicolumn{13}{l}{A dash (-) indicates a timeout instance with no feasible span found.}
\end{tabular}
}
\end{table*}

\begin{table*}[ht!]
\renewcommand{\arraystretch}{1.25}
\setlength{\tabcolsep}{3pt}
\centering
\caption{Experimental results with coefficient 1.25.}
\label{tab:experimental-results-1.25}
\resizebox{\textwidth}{!}{
\begin{tabular}{lcp{0.25cm}cccccp{0.25cm}cccc}
\hline
\multicolumn{1}{c}{\multirow{2}{*}{\textbf{Graph}}} & \multirow{2}{*}{\textbf{|V|}} &  & \multicolumn{5}{c}{\textbf{MSABL}} &  & \multicolumn{4}{c}{\textbf{MSCABL}} \\ \cline{4-8} \cline{10-13} 
\multicolumn{1}{c}{} &  &  & \textbf{CPLEX\textsubscript{CP}} & \textbf{CPLEX\textsubscript{MIP}} & \textbf{Gurobi} & \textbf{SAT\textsubscript{Par}} & \textbf{SAT\textsubscript{Inc}} &  & \textbf{CPLEX\textsubscript{CP}} & \textbf{CPLEX\textsubscript{MIP}} & \textbf{Gurobi} & \textbf{SAT\textsubscript{Par}} \\ \hline
A-pores\_1 & 30 &  & \opt{21} & \opt{21} & \opt{21} & \opt{21} & \opt{21} &  & \opt{28} & \opt{28} & \opt{28} & \opt{28} \\
B-ibm32 & 32 &  & \opt{33} & \opt{33} & \opt{33} & \opt{33} & \opt{33} &  & \opt{39} & \opt{39} & \opt{39} & \opt{39} \\
C-bcspwr01 & 39 &  & \opt{42} & \opt{42} & \opt{42} & \opt{42} & \opt{42} &  & \opt{47} & \opt{47} & \opt{47} & \opt{47} \\
D-bcsstk01 & 48 &  & \opt{55} & 55 & \opt{55} & \opt{55} & \opt{55} &  & \opt{59} & \opt{59} & \opt{59} & \opt{59} \\
E-bcspwr02 & 49 &  & \opt{52} & \opt{52} & \opt{52} & \opt{52} & \opt{52} &  & \opt{59} & \opt{59} & \opt{59} & \opt{59} \\
F-curtis54 & 54 &  & \opt{64} & \opt{64} & \opt{64} & \opt{64} & \opt{64} &  & \opt{59} & \opt{59} & 59 & \opt{59} \\
G-will57 & 57 &  & \opt{64} & \opt{64} & \opt{64} & \opt{64} & \opt{64} &  & \opt{64} & \opt{64} & 64 & \opt{64} \\
H-impcol\_b & 59 &  & 70 & 70 & \opt{70} & \opt{70} & \opt{70} &  & 63 & 63 & 63 & \textbf{\opt{63}} \\
I-ash85 & 85 &  & \opt{84} & \opt{84} & \opt{84} & \opt{84} & \opt{84} &  & \opt{103} & \opt{103} & \opt{103} & \opt{103} \\
J-nos4 & 100 &  & \opt{86} & \opt{86} & \opt{86} & \opt{86} & \opt{86} &  & \opt{119} & \opt{119} & \opt{119} & \opt{119} \\
K-dwt\_\_234 & 117 &  & \opt{126} & \opt{126} & \opt{126} & \opt{126} & \opt{126} &  & \opt{136} & \opt{136} & \opt{136} & \opt{136} \\
L-bcspwr03 & 118 &  & \opt{144} & \opt{144} & \opt{144} & \opt{144} & \opt{144} &  & \opt{143} & 143 & \opt{143} & \opt{143} \\
M-bcsstk06 & 420 &  & 462 & - & 462 & 464 & \textbf{\opt{462}} &  & \textbf{491} & - & - & 494 \\
N-bcsstk07 & 420 &  & 462 & - & 462 & 464 & \textbf{\opt{462}} &  & \textbf{491} & - & - & 494 \\
O-impcol\_d & 425 &  & \opt{450} & - & \opt{450} & \opt{450} & \opt{450} &  & \opt{523} & - & - & \opt{523} \\
P-can\_\_445 & 445 &  & 448 & - & 448 & \textbf{\opt{448}} & \textbf{\opt{448}} &  & 460 & - & - & \textbf{455} \\
Q-494\_bus & 494 &  & \opt{566} & - & \opt{566} & \opt{566} & \opt{566} &  & \opt{614} & 616 & \opt{614} & \opt{614} \\
R-dwt\_\_503 & 503 &  & 546 & - & 624 & \textbf{\opt{546}} & \textbf{\opt{546}} &  & 615 & - & - & \textbf{\opt{615}} \\
S-sherman4 & 546 &  & \opt{326} & \opt{326} & \opt{326} & \opt{326} & \opt{326} &  & \opt{644} & - & \opt{644} & \opt{644} \\
T-dwt\_\_592 & 592 &  & \opt{423} & - & \opt{423} & \opt{423} & \opt{423} &  & \opt{564} & - & - & \opt{564} \\
U-662\_bus & 662 &  & \opt{825} & - & 825 & \opt{825} & \opt{825} &  & \opt{823} & - & 823 & \opt{823} \\
V-nos6 & 675 &  & \opt{411} & - & \opt{411} & \opt{411} & \opt{411} &  & \opt{820} & - & \opt{820} & \opt{820} \\
W-685\_bus & 685 &  & 850 & - & 850 & \textbf{\opt{850}} & \textbf{\opt{850}} &  & 851 & - & 851 & \textbf{\opt{851}} \\
X-can\_\_715 & 715 &  & 725 & - & 725 & \textbf{\opt{725}} & \textbf{\opt{725}} &  & 755 & - & - & \textbf{\opt{755}} \\ \hline
\multicolumn{2}{l}{$\#Solved$} &  & 24 & 13 & 24 & 24 & 24 &  & 24 & 13 & 17 & 24 \\
\multicolumn{2}{l}{$\#Optimal$} &  & 17 & 11 & 17 & 22 & \textbf{24} &  & 17 & 10 & 12 & \textbf{21} \\
\multicolumn{2}{l}{$\#Best$} &  & 17 & 11 & 17 & 22 & \textbf{24} &  & 19 & 10 & 12 & \textbf{22} \\
\multicolumn{2}{l}{$\#Avg.Rank$} &  & 2.94 & 4.06 & 2.96 & 2.65 & \textbf{2.4} &  & 2.08 & 3.15 & 2.88 & \textbf{1.9} \\ \hline
\multicolumn{13}{l}{The best result for each instance and comparison criterion is highlighted in bold.} \\
\multicolumn{13}{l}{An asterisk (*) indicates an optimal result.} \\
\multicolumn{13}{l}{A dash (-) indicates a timeout instance with no feasible span found.}
\end{tabular}
}
\end{table*}

\begin{table*}[ht!]
\renewcommand{\arraystretch}{1.25}
\setlength{\tabcolsep}{3pt}
\centering
\caption{Experimental results with coefficient 1.5.}
\label{tab:experimental-results-1.5}
\resizebox{\textwidth}{!}{
\begin{tabular}{lcp{0.25cm}cccccp{0.25cm}cccc}
\hline
\multicolumn{1}{c}{\multirow{2}{*}{\textbf{Graph}}} & \multirow{2}{*}{\textbf{|V|}} &  & \multicolumn{5}{c}{\textbf{MSABL}} &  & \multicolumn{4}{c}{\textbf{MSCABL}} \\ \cline{4-8} \cline{10-13} 
\multicolumn{1}{c}{} &  &  & \textbf{CPLEX\textsubscript{CP}} & \textbf{CPLEX\textsubscript{MIP}} & \textbf{Gurobi} & \textbf{SAT\textsubscript{Par}} & \textbf{SAT\textsubscript{Inc}} &  & \textbf{CPLEX\textsubscript{CP}} & \textbf{CPLEX\textsubscript{MIP}} & \textbf{Gurobi} & \textbf{SAT\textsubscript{Par}} \\ \hline
A-pores\_1 & 30 &  & \opt{27} & \opt{27} & \opt{27} & \opt{27} & \opt{27} &  & \opt{36} & \opt{36} & \opt{36} & \opt{36} \\
B-ibm32 & 32 &  & \opt{39} & \opt{39} & \opt{39} & \opt{39} & \opt{39} &  & \opt{47} & \opt{47} & \opt{47} & \opt{47} \\
C-bcspwr01 & 39 &  & \opt{50} & \opt{50} & \opt{50} & \opt{50} & \opt{50} &  & \opt{56} & \opt{56} & \opt{56} & \opt{56} \\
D-bcsstk01 & 48 &  & \opt{65} & - & \opt{65} & \opt{65} & \opt{65} &  & \opt{71} & \opt{71} & \opt{71} & \opt{71} \\
E-bcspwr02 & 49 &  & \opt{62} & \opt{62} & \opt{62} & \opt{62} & \opt{62} &  & \opt{71} & \opt{71} & \opt{71} & \opt{71} \\
F-curtis54 & 54 &  & \opt{76} & - & \opt{76} & \opt{76} & \opt{76} &  & \opt{74} & \opt{74} & \opt{74} & \opt{74} \\
G-will57 & 57 &  & \opt{76} & - & \opt{76} & \opt{76} & \opt{76} &  & \opt{79} & 79 & \opt{79} & \opt{79} \\
H-impcol\_b & 59 &  & 84 & - & \opt{84} & \opt{84} & \opt{84} &  & 79 & 79 & 79 & \textbf{\opt{79}} \\
I-ash85 & 85 &  & \opt{102} & - & \opt{102} & \opt{102} & \opt{102} &  & \opt{123} & 123 & \opt{123} & \opt{123} \\
J-nos4 & 100 &  & \opt{104} & \opt{104} & \opt{104} & \opt{104} & \opt{104} &  & \opt{143} & 146 & \opt{143} & \opt{143} \\
K-dwt\_\_234 & 117 &  & \opt{152} & \opt{152} & \opt{152} & \opt{152} & \opt{152} &  & \opt{165} & \opt{165} & \opt{165} & \opt{165} \\
L-bcspwr03 & 118 &  & \opt{174} & - & \opt{174} & \opt{174} & \opt{174} &  & \opt{171} & 171 & 171 & \opt{171} \\
M-bcsstk06 & 420 &  & 561 & - & 612 & 561 & 561 &  & \textbf{587} & - & - & 593 \\
N-bcsstk07 & 420 &  & 561 & - & 612 & 561 & 561 &  & \textbf{587} & - & - & 593 \\
O-impcol\_d & 425 &  & \opt{540} & - & \opt{540} & \opt{540} & \opt{540} &  & \opt{627} & - & - & \opt{627} \\
P-can\_\_445 & 445 &  & 540 & - & 540 & \textbf{\opt{540}} & \textbf{\opt{540}} &  & 649 & - & - & \textbf{549} \\
Q-494\_bus & 494 &  & \opt{680} & - & \opt{680} & \opt{680} & \opt{680} &  & \opt{737} & - & \opt{737} & \opt{737} \\
R-dwt\_\_503 & 503 &  & 658 & - & 752 & \textbf{\opt{658}} & \textbf{\opt{658}} &  & 743 & - & - & \textbf{\opt{743}} \\
S-sherman4 & 546 &  & \opt{391} & - & \opt{391} & \opt{391} & \opt{391} &  & \opt{774} & - & \opt{774} & \opt{774} \\
T-dwt\_\_592 & 592 &  & \opt{507} & - & \opt{507} & \opt{507} & \opt{507} &  & \opt{676} & - & - & \opt{676} \\
U-662\_bus & 662 &  & \opt{990} & - & 990 & \opt{990} & \opt{990} &  & \opt{987} & - & 987 & \opt{987} \\
V-nos6 & 675 &  & \opt{493} & - & \opt{493} & \opt{493} & \opt{493} &  & \opt{984} & 989 & \opt{984} & \opt{984} \\
W-685\_bus & 685 &  & 1020 & - & \opt{1020} & \opt{1020} & \opt{1020} &  & 1025 & - & - & \textbf{\opt{1025}} \\
X-can\_\_715 & 715 &  & 870 & - & 870 & \textbf{\opt{870}} & \textbf{\opt{870}} &  & 905 & - & - & \textbf{\opt{905}} \\ \hline
\multicolumn{2}{l}{$\#Solved$} &  & 24 & 6 & 24 & 24 & 24 &  & 24 & 13 & 16 & 24 \\
\multicolumn{2}{l}{$\#Optimal$} &  & 17 & 6 & 18 & \textbf{22} & \textbf{22} &  & 17 & 7 & 13 & \textbf{21} \\
\multicolumn{2}{l}{$\#Best$} &  & 19 & 6 & 18 & \textbf{24} & \textbf{24} &  & 19 & 7 & 13 & \textbf{22} \\
\multicolumn{2}{l}{$\#Avg.Rank$} &  & 2.79 & 4.5 & 2.92 & \textbf{2.4} & \textbf{2.4} &  & 2.02 & 3.33 & 2.79 & \textbf{1.85} \\ \hline
\multicolumn{13}{l}{The best result for each instance and comparison criterion is highlighted in bold.} \\
\multicolumn{13}{l}{An asterisk (*) indicates an optimal result.} \\
\multicolumn{13}{l}{A dash (-) indicates a timeout instance with no feasible span found.}
\end{tabular}
}
\end{table*}

\begin{table*}[ht!]
\renewcommand{\arraystretch}{1.25}
\setlength{\tabcolsep}{3pt}
\centering
\caption{Experimental results with coefficient 0.5 on no-hole model.}
\label{tab:experimental-results-0.5_no_hole}
\resizebox{\textwidth}{!}{
\begin{tabular}{lcp{0.25cm}cccccp{0.25cm}cccc}
\hline
\multicolumn{1}{c}{\multirow{2}{*}{\textbf{Graph}}} & \multirow{2}{*}{\textbf{|V|}} &  & \multicolumn{5}{c}{\textbf{MSABL}} &  & \multicolumn{4}{c}{\textbf{MSCABL}} \\ \cline{4-8} \cline{10-13} 
\multicolumn{1}{c}{} &  &  & \textbf{CPLEX\textsubscript{CP}} & \textbf{CPLEX\textsubscript{MIP}} & \textbf{Gurobi} & \textbf{SAT\textsubscript{Par}} & \textbf{SAT\textsubscript{Inc}} &  & \textbf{CPLEX\textsubscript{CP}} & \textbf{CPLEX\textsubscript{MIP}} & \textbf{Gurobi} & \textbf{SAT\textsubscript{Par}} \\ \hline
A-pores\_1 & 30 &  & \opt{11} & \opt{11} & \opt{11} & \opt{11} & \opt{11} &  & \opt{12} & \opt{12} & \opt{12} & \opt{12} \\
B-ibm32 & 32 &  & \opt{12} & \opt{12} & \opt{12} & \opt{12} & \opt{12} &  & \opt{15} & \opt{15} & \opt{15} & \opt{15} \\
C-bcspwr01 & 39 &  & \opt{16} & \opt{16} & \opt{16} & \opt{16} & \opt{16} &  & \opt{17} & \opt{17} & \opt{17} & \opt{17} \\
D-bcsstk01 & 48 &  & \opt{20} & 20 & \opt{20} & \opt{20} & \opt{20} &  & \opt{23} & \opt{23} & \opt{23} & \opt{23} \\
E-bcspwr02 & 49 &  & \opt{20} & \opt{20} & \opt{20} & \opt{20} & \opt{20} &  & \opt{23} & \opt{23} & \opt{23} & \opt{23} \\
F-curtis54 & 54 &  & \opt{24} & 24 & \opt{24} & \opt{24} & \opt{24} &  & \opt{24} & \opt{24} & \opt{24} & \opt{24} \\
G-will57 & 57 &  & \opt{24} & \opt{24} & \opt{24} & \opt{24} & \opt{24} &  & \opt{24} & \opt{24} & \opt{24} & \opt{24} \\
H-impcol\_b & 59 &  & \opt{28} & 28 & 28 & \opt{28} & \opt{28} &  & \opt{23} & 23 & 23 & \opt{23} \\
I-ash85 & 85 &  & 34 & 38 & 37 & \textbf{\opt{33}} & \textbf{\opt{33}} &  & \opt{39} & 41 & \opt{39} & \opt{39} \\
J-nos4 & 100 &  & 41 & 41 & 41 & \textbf{\opt{41}} & \textbf{\opt{41}} &  & 49 & 49 & 49 & 49 \\
K-dwt\_\_234 & 117 &  & \opt{50} & \opt{50} & \opt{50} & \opt{50} & \opt{50} &  & \opt{55} & \opt{55} & \opt{55} & \opt{55} \\
L-bcspwr03 & 118 &  & \opt{57} & \opt{57} & 57 & \opt{57} & \opt{57} &  & \opt{55} & \opt{55} & \opt{55} & \opt{55} \\
M-bcsstk06 & 420 &  & \textbf{189} & - & - & 191 & 190 &  & 193 & - & - & \textbf{191} \\
N-bcsstk07 & 420 &  & \textbf{189} & - & - & 191 & 190 &  & 193 & - & - & \textbf{191} \\
O-impcol\_d & 425 &  & \textbf{\opt{180}} & - & - & 192 & 186 &  & \opt{207} & - & - & \opt{207} \\
P-can\_\_445 & 445 &  & 202 & - & - & \textbf{199} & 200 &  & 218 & - & - & \textbf{181} \\
Q-494\_bus & 494 &  & \opt{226} & - & 231 & \opt{226} & \opt{226} &  & \opt{245} & - & - & \opt{245} \\
R-dwt\_\_503 & 503 &  & 217 & - & - & \textbf{\opt{217}} & \textbf{\opt{217}} &  & 247 & - & - & \textbf{\opt{247}} \\
S-sherman4 & 546 &  & \opt{259} & - & - & \opt{259} & \opt{259} &  & 267 & - & - & \textbf{\opt{266}} \\
T-dwt\_\_592 & 592 &  & 254 & - & - & \textbf{\opt{223}} & 241 &  & 283 & - & - & \textbf{229} \\
U-662\_bus & 662 &  & \opt{330} & - & - & \opt{330} & \opt{330} &  & \opt{327} & - & - & \opt{327} \\
V-nos6 & 675 &  & 328 & - & - & \textbf{\opt{327}} & \textbf{\opt{327}} &  & 336 & - & - & \textbf{334} \\
W-685\_bus & 685 &  & 340 & - & - & \textbf{\opt{340}} & \textbf{\opt{340}} &  & 341 & - & - & \textbf{\opt{341}} \\
X-can\_\_715 & 715 &  & 308 & - & - & 297 & \textbf{292} &  & 299 & - & - & \textbf{\opt{299}} \\ \hline
\multicolumn{2}{l}{$\#Solved$} &  & 24 & 12 & 13 & 24 & 24 &  & 24 & 12 & 12 & 24 \\
\multicolumn{2}{l}{$\#Optimal$} &  & 14 & 7 & 8 & \textbf{19} & 18 &  & 14 & 9 & 10 & \textbf{18} \\
\multicolumn{2}{l}{$\#Best$} &  & 16 & 7 & 8 & \textbf{20} & 19 &  & 15 & 10 & 11 & \textbf{24} \\
\multicolumn{2}{l}{$\#Avg.Rank$} &  & 2.56 & 4.1 & 3.92 & 2.25 & \textbf{2.17} &  & 2.13 & 3.1 & 3.02 & \textbf{1.75} \\ \hline
\multicolumn{13}{l}{The best result for each instance and comparison criterion is highlighted in bold.} \\
\multicolumn{13}{l}{An asterisk (*) indicates an optimal result.} \\
\multicolumn{13}{l}{A dash (-) indicates a timeout instance with no feasible span found.}
\end{tabular}
}
\end{table*}

\begin{table*}[ht!]
\renewcommand{\arraystretch}{1.25}
\setlength{\tabcolsep}{3pt}
\centering
\caption{Experimental results with coefficient 0.75 on no-hole model.}
\label{tab:experimental-results-0.75-no-hole}
\resizebox{\textwidth}{!}{
\begin{tabular}{lcp{0.25cm}cccccp{0.25cm}cccc}
\hline
\multicolumn{1}{c}{\multirow{2}{*}{\textbf{Graph}}} & \multirow{2}{*}{\textbf{|V|}} &  & \multicolumn{5}{c}{\textbf{MSABL}} &  & \multicolumn{4}{c}{\textbf{MSCABL}} \\ \cline{4-8} \cline{10-13} 
\multicolumn{1}{c}{} &  &  & \textbf{CPLEX\textsubscript{CP}} & \textbf{CPLEX\textsubscript{MIP}} & \textbf{Gurobi} & \textbf{SAT\textsubscript{Par}} & \textbf{SAT\textsubscript{Inc}} &  & \textbf{CPLEX\textsubscript{CP}} & \textbf{CPLEX\textsubscript{MIP}} & \textbf{Gurobi} & \textbf{SAT\textsubscript{Par}} \\ \hline
A-pores\_1 & 30 &  & \opt{15} & \opt{15} & \opt{15} & \opt{15} & \opt{15} &  & \opt{17} & \opt{17} & \opt{17} & \opt{17} \\
B-ibm32 & 32 &  & \opt{18} & \opt{18} & \opt{18} & \opt{18} & \opt{18} &  & \opt{23} & \opt{23} & \opt{23} & \opt{23} \\
C-bcspwr01 & 39 &  & \opt{24} & \opt{24} & \opt{24} & \opt{24} & \opt{24} &  & \opt{26} & \opt{26} & \opt{26} & \opt{26} \\
D-bcsstk01 & 48 &  & \opt{30} & \opt{30} & 30 & \opt{30} & \opt{30} &  & \opt{35} & \opt{35} & 35 & \opt{35} \\
E-bcspwr02 & 49 &  & \opt{30} & \opt{30} & \opt{30} & \opt{30} & \opt{30} &  & \opt{35} & \opt{35} & \opt{35} & \opt{35} \\
F-curtis54 & 54 &  & \opt{36} & 36 & \opt{36} & \opt{36} & \opt{36} &  & \opt{34} & \opt{34} & \opt{34} & \opt{34} \\
G-will57 & 57 &  & \opt{36} & \opt{36} & \opt{36} & \opt{36} & \opt{36} &  & \opt{39} & \opt{39} & \opt{39} & \opt{39} \\
H-impcol\_b & 59 &  & \opt{42} & 42 & 42 & \opt{42} & \opt{42} &  & \opt{39} & 39 & 39 & \opt{39} \\
I-ash85 & 85 &  & 58 & - & - & \textbf{57} & \textbf{57} &  & \opt{59} & 62 & 62 & \opt{59} \\
J-nos4 & 100 &  & 68 & 72 & 69 & \textbf{\opt{68}} & \textbf{\opt{68}} &  & 74 & - & 74 & 74 \\
K-dwt\_\_234 & 117 &  & \opt{76} & 81 & 86 & \opt{76} & \opt{76} &  & 83 & 86 & 83 & \textbf{\opt{83}} \\
L-bcspwr03 & 118 &  & \opt{87} & 87 & 87 & \opt{87} & \opt{87} &  & \opt{83} & 83 & 83 & \opt{83} \\
M-bcsstk06 & 420 &  & 307 & - & - & 296 & \textbf{289} &  & 312 & - & - & \textbf{296} \\
N-bcsstk07 & 420 &  & 307 & - & - & 296 & \textbf{289} &  & 312 & - & - & \textbf{296} \\
O-impcol\_d & 425 &  & \textbf{292} & - & - & 310 & 306 &  & \opt{311} & - & - & \opt{311} \\
P-can\_\_445 & 445 &  & \textbf{307} & - & - & 311 & 310 &  & 328 & - & - & \textbf{275} \\
Q-494\_bus & 494 &  & \textbf{\opt{340}} & - & - & - & - &  & \opt{368} & - & - & \opt{368} \\
R-dwt\_\_503 & 503 &  & 366 & - & - & \textbf{345} & - &  & 375 & - & - & 375 \\
S-sherman4 & 546 &  & - & - & - & - & - &  & - & - & - & \textbf{402} \\
T-dwt\_\_592 & 592 &  & 399 & - & - & \textbf{386} & 424 &  & 423 & - & - & \textbf{344} \\
U-662\_bus & 662 &  & \opt{495} & - & - & \opt{495} & \opt{495} &  & \opt{491} & - & - & \opt{491} \\
V-nos6 & 675 &  & - & - & - & - & - &  & - & - & - & - \\
W-685\_bus & 685 &  & 510 & - & - & \textbf{\opt{510}} & \textbf{\opt{510}} &  & 509 & - & - & \textbf{\opt{509}} \\
X-can\_\_715 & 715 &  & \textbf{469} & - & - & 493 & - &  & 456 & - & - & \textbf{455} \\ \hline
\multicolumn{2}{l}{$\#Solved$} &  & \textbf{22} & 11 & 11 & 21 & 19 &  & 22 & 11 & 12 & \textbf{23} \\
\multicolumn{2}{l}{$\#Optimal$} &  & 12 & 6 & 6 & \textbf{13} & \textbf{13} &  & 13 & 7 & 6 & \textbf{15} \\
\multicolumn{2}{l}{$\#Best$} &  & 15 & 6 & 6 & \textbf{16} & \textbf{16} &  & 15 & 7 & 7 & \textbf{23} \\
\multicolumn{2}{l}{$\#Avg.Rank$} &  & 2.38 & 3.92 & 3.92 & \textbf{2.33} & 2.46 &  & 2.06 & 3.17 & 3.1 & \textbf{1.67} \\ \hline
\multicolumn{13}{l}{The best result for each instance and comparison criterion is highlighted in bold.} \\
\multicolumn{13}{l}{An asterisk (*) indicates an optimal result.} \\
\multicolumn{13}{l}{A dash (-) indicates a timeout instance with no feasible span found.}
\end{tabular}
}
\end{table*}

\begin{table*}[ht!]
\renewcommand{\arraystretch}{1.25}
\setlength{\tabcolsep}{3pt}
\centering
\caption{Experimental results with coefficient 1 on no-hole model.}
\label{tab:experimental-results-1-no-hole}
\resizebox{\textwidth}{!}{
\begin{tabular}{lcp{0.25cm}cccccp{0.25cm}cccc}
\hline
\multicolumn{1}{c}{\multirow{2}{*}{\textbf{Graph}}} & \multirow{2}{*}{\textbf{|V|}} &  & \multicolumn{5}{c}{\textbf{MSABL}} &  & \multicolumn{4}{c}{\textbf{MSCABL}} \\ \cline{4-8} \cline{10-13} 
\multicolumn{1}{c}{} &  &  & \textbf{CPLEX\textsubscript{CP}} & \textbf{CPLEX\textsubscript{MIP}} & \textbf{Gurobi} & \textbf{SAT\textsubscript{Par}} & \textbf{SAT\textsubscript{Inc}} &  & \textbf{CPLEX\textsubscript{CP}} & \textbf{CPLEX\textsubscript{MIP}} & \textbf{Gurobi} & \textbf{SAT\textsubscript{Par}} \\ \hline
A-pores\_1 & 30 &  & \opt{23} & \opt{23} & \opt{23} & \opt{23} & \opt{23} &  & \opt{26} & \opt{26} & \opt{26} & \opt{26} \\
B-ibm32 & 32 &  & \opt{29} & 29 & \opt{29} & \opt{29} & \opt{29} &  & \opt{31} & \opt{31} & \opt{31} & \opt{31} \\
C-bcspwr01 & 39 &  & 38 & \textbf{\opt{38}} & 38 & 38 & 38 &  & \opt{38} & \opt{38} & \opt{38} & \opt{38} \\
D-bcsstk01 & 48 &  & \textbf{\opt{46}} & - & - & 46 & 46 &  & \opt{47} & - & \opt{47} & \opt{47} \\
E-bcspwr02 & 49 &  & 45 & \textbf{\opt{45}} & 45 & 45 & 45 &  & \opt{47} & \opt{47} & \opt{47} & \opt{47} \\
F-curtis54 & 54 &  & \textbf{\opt{53}} & - & - & 53 & 53 &  & \opt{49} & 49 & \opt{49} & \opt{49} \\
G-will57 & 57 &  & 54 & - & 54 & 54 & 54 &  & \opt{54} & 54 & \opt{54} & \opt{54} \\
H-impcol\_b & 59 &  & \opt{56} & - & - & \opt{56} & \opt{56} &  & 55 & - & 55 & \textbf{\opt{55}} \\
I-ash85 & 85 &  & - & - & - & \textbf{-} & \textbf{-} &  & \textbf{\opt{84}} & - & - & - \\
J-nos4 & 100 &  & - & - & - & \textbf{-} & \textbf{-} &  & \textbf{99} & - & - & - \\
K-dwt\_\_234 & 117 &  & - & - & - & - & - &  & \textbf{116} & - & - & \textbf{-} \\
L-bcspwr03 & 118 &  & \opt{117} & - & - & \opt{117} & \opt{117} &  & \opt{115} & - & 116 & \opt{115} \\
M-bcsstk06 & 420 &  & - & - & - & - & \textbf{-} &  & - & - & - & \textbf{-} \\
N-bcsstk07 & 420 &  & - & - & - & - & \textbf{-} &  & - & - & - & \textbf{-} \\
O-impcol\_d & 425 &  & \textbf{-} & - & - & - & - &  & \textbf{422} & - & - & - \\
P-can\_\_445 & 445 &  & \textbf{-} & - & - & - & - &  & 440 & - & - & \textbf{377} \\
Q-494\_bus & 494 &  & \textbf{-} & - & - & - & - &  & \textbf{\opt{493}} & - & - & - \\
R-dwt\_\_503 & 503 &  & - & - & - & \textbf{-} & - &  & - & - & - & - \\
S-sherman4 & 546 &  & - & - & - & - & - &  & - & - & - & \textbf{-} \\
T-dwt\_\_592 & 592 &  & \textbf{551} & - & - & \textbf{-} & - &  & \textbf{581} & - & - & \textbf{-} \\
U-662\_bus & 662 &  & \textbf{\opt{661}} & - & - & - & - &  & \textbf{\opt{659}} & - & - & - \\
V-nos6 & 675 &  & - & - & - & - & - &  & \textbf{674} & - & - & - \\
W-685\_bus & 685 &  & \textbf{680} & - & - & \textbf{-} & \textbf{-} &  & \textbf{683} & - & - & \textbf{-} \\
X-can\_\_715 & 715 &  & \textbf{679} & - & - & - & - &  & \textbf{621} & - & - & \textbf{-} \\ \hline
\multicolumn{2}{l}{$\#Solved$} &  & \textbf{13} & 4 & 5 & 9 & 9 &  & \textbf{20} & 6 & 9 & 10 \\
\multicolumn{2}{l}{$\#Optimal$} &  & \textbf{7} & 3 & 2 & 4 & 4 &  & \textbf{11} & 4 & 7 & 9 \\
\multicolumn{2}{l}{$\#Best$} &  & \textbf{11} & 3 & 3 & 5 & 5 &  & \textbf{18} & 4 & 7 & 10 \\
\multicolumn{2}{l}{$\#Avg.Rank$} &  & \textbf{2.42} & 3.33 & 3.33 & 2.96 & 2.96 &  & \textbf{1.75} & 3.06 & 2.71 & 2.48 \\ \hline
\multicolumn{13}{l}{The best result for each instance and comparison criterion is highlighted in bold.} \\
\multicolumn{13}{l}{An asterisk (*) indicates an optimal result.} \\
\multicolumn{13}{l}{A dash (-) indicates a timeout instance with no feasible span found.}
\end{tabular}
}
\end{table*}
\end{appendices}


\clearpage

\bibliography{sn-bibliography}

\end{document}